\documentclass{article}

\usepackage{arxiv}

\usepackage{epstopdf}
\usepackage{amsthm}
\usepackage{hyperref}
\usepackage{url}
\usepackage{hyperref}
\usepackage{url}
\usepackage[utf8]{inputenc} 
\usepackage[T1]{fontenc}    
\usepackage{booktabs}       
\usepackage{amsfonts}       
\usepackage{nicefrac}       
\usepackage{microtype}      
\usepackage{mathtools}
\usepackage{amssymb}
\usepackage[shortlabels]{enumitem}
\usepackage{xcolor,colortbl}
\usepackage{comment}
\usepackage{subcaption}
\usepackage{graphicx}
\usepackage{makecell}
\usepackage{multirow}
\usepackage{float}
\usepackage{booktabs,caption}
\usepackage[flushleft]{threeparttable}
\usepackage{ulem}
\usepackage{algorithm}
\usepackage{algpseudocode}
\usepackage[T1]{fontenc}
\usepackage{lmodern}

\newtheorem{theorem}{Theorem}
\newtheorem{definition}[theorem]{Definition}
\newtheorem{remark}[theorem]{Remark}
\newtheorem{lemma}[theorem]{Lemma}
\newtheorem{corollary}[theorem]{Corollary}

\numberwithin{theorem}{section}

\title{Optimal Transport Regularized Divergences: Application to Adversarial Robustness}

\author{Jeremiah Birrell\\
  Department of Mathematics\\
  Texas State University\\
  San Marcos, TX,  USA \\
  \texttt{jbirrell@txstate.edu} \\
   \And
Reza Ebrahimi\\
  School of Information Systems and Management\\
University of South Florida\\
  Tampla, FL, USA \\
  \texttt{ebrahimim@usf.edu} 
}

\begin{document}

\maketitle

\begin{abstract}
We introduce a new class of optimal-transport-regularized divergences, $D^c$, constructed via an infimal convolution between an information divergence, $D$, and an optimal-transport (OT) cost, $C$, and study their use in distributionally robust optimization (DRO).  In particular,  we propose the $ARMOR_D$ methods as novel approaches to enhancing the adversarial robustness of deep learning models. These  DRO-based methods are defined by minimizing the maximum expected loss over a $D^c$-neighborhood of the empirical distribution of the training data.   Viewed as a tool for constructing adversarial samples,  our method allows samples to be both transported, according to the OT cost, and re-weighted, according to the information divergence; the addition of a principled and dynamical adversarial re-weighting on top of adversarial sample transport is a key innovation of $ARMOR_D$.  $ARMOR_D$ can be viewed as a generalization of the best-performing loss functions and OT costs in the adversarial training literature; we demonstrate this flexibility by using $ARMOR_D$ to augment the UDR, TRADES, and MART methods and obtain  improved performance on CIFAR-10 and CIFAR-100 image recognition. Specifically, augmenting with $ARMOR_D$ leads to 1.9\% and 2.1\% improvement  against  AutoAttack, a powerful ensemble of adversarial attacks, on CIFAR-10 and CIFAR-100 respectively.  To foster reproducibility, we made the code accessible at \url{https://github.com/star-ailab/ARMOR}.
\end{abstract}


\section{Introduction}
Machine learning and specifically deep learning models are known to be vulnerable to adversarial samples: inputs intentionally and meticulously modified by an adversary to evade or mislead the classification model \cite{papernot2016limitations,goodfellow2014explaining}.  To date, the large body of prominent defense mechanisms for mitigating this vulnerability and enhancing adversarial robustness includes  adversarial training methods \cite{papernot2017mask,madry2018towardsrobustoptimization,hu2018does,wang2019improving_mart,zhang2019theoreticallytradeOff,zhang2020geometry,dong2020adversarial,regniez2021distributional,bui_UDR_2022unified,dong2023towards} which  construct adversarial samples that are employed during training, and  certifiable approaches \cite{cert_baharlouei2023improving,cert_raghunathan2018certified}, which can  guarantee the absence of adversarial examples misclassified by the model for a specific input, with \cite{sinha2018certifying}   having aspects of both categories. Despite attractive guarantees, the  certifiable approaches often operate on a convex relaxation of the original model rather than the original model and tend to have inferior performance compared to adversarial training \cite{wang2019improving_mart, athalye2018obfuscatedfalsesense}. Adversarial training also remains challenging, as it is difficult to maintain the model's performance generalizability while also  enhancing its adversarial robustness \cite{carlini2019evaluating, zhang2019theoreticallytradeOff}.


In PGD \cite{madry2018towardsrobustoptimization}, the pioneering robust optimization approach to adversarial training,   the loss function $\mathcal{L}_\theta$, depending on parameters $\theta\in\Theta$, is maximized over a metric-space ball centered at the training samples $z_i$, leading to the empirical risk minimization problem
\begin{align}
\inf_\theta E_{P_n}\left[ \sup_{\tilde z:d(z,\tilde z)\leq \epsilon}\mathcal{L}_\theta(\tilde z)\right]\,,\label{eq:classical_robust_optimization}
\end{align}
where $P_n=\frac{1}{n}\sum_{i=1}^n \delta_{z_i}$ is the empirical distribution. Solving the inner maximization problem in \eqref{eq:classical_robust_optimization} involves the construction of adversarial samples, $\tilde z_i\in \mathrm{argmax}_{\tilde z:d(z_i,\tilde z)\leq \epsilon} \mathcal{L}_\theta(\tilde z)$, corresponding to the true samples, $z_i$.  In \cite{bui_UDR_2022unified} it was noted that a large class of robust classification methods, including \eqref{eq:classical_robust_optimization}, can be expressed as distributionally robust optimization (DRO) problems. In DRO,  a stochastic optimization problem, $\inf_\theta E_P[\mathcal{L}_\theta]$, is regularized (or robustified) via the introduction of an adversarial maximization over a neighborhood of distributions, $\mathcal{U}(P)$, around the baseline distribution $P$:
\begin{align}\label{eq:DRO_gen}
    \inf_\theta \sup_{Q\in\mathcal{U}(P)}E_Q[\mathcal{L}_\theta]\,.
\end{align}
This DRO problem formalizes an uncertainty in the underlying distribution $P$ and can protect against overfitting, leading to better out-of-sample performance; see \cite{Rahimian_2022} for an overview of DRO. For general distribution neighborhoods \eqref{eq:DRO_gen} is an intractable infinite dimensional problem but if $\mathcal{U}$ has the appropriate structure then one can derive tractable finite dimensional reformulations of  \eqref{eq:DRO_gen}.    Prior  approaches to the general theory of DRO employ various types  of distribution neighborhoods, such as moment constraints \cite{doi:10.1287/opre.1090.0795,doi:10.1287/opre.1090.0741,doi:10.1287/opre.2014.1314}, conditional moment constraints \cite{2023arXiv230805414B}, MMD \cite{NEURIPS2019_1770ae9e}, Kullback-Leibler (KL)  and  $f$-divergence neighborhoods \cite{doi:10.1287/opre.1100.0821,Javid2012,Hu2013,10.2307/23359484,doi:10.1287/opre.2018.1786}, smoothed $f$-divergences \cite{liu2023smoothed},   Wasserstein neighborhoods \cite{EsfahaniKuhn2018,JMLR:v20:17-633,wu2022generalization,li2022general,doi:10.1287/moor.2022.1275}, Sinkhorn diveregence \cite{wang2021sinkhorn}, and more general optimal-transport (OT) neighborhoods \cite{doi:10.1287/moor.2018.0936,azizian2023regularization}.

In \cite{bui_UDR_2022unified} it was shown that the adversarial training methods PGD   \cite{madry2018towardsrobustoptimization}, TRADES,  \cite{zhang2019theoreticallytradeOff}, and MART \cite{wang2019improving_mart}, can  all be formulated as DRO problems of the type
\begin{align}\label{eq:OT_DRO_Bui}
&\inf_\theta\sup_{Q:C_d(Q,P_n)\leq \epsilon}E_Q[\mathcal{L}_\theta]\,, \,\,\,\,\,c_d(z,\tilde z)=\infty 1_{d(x,\tilde x)>\epsilon}+\infty 1_{x\neq \tilde{x}^\prime}+\infty 1_{y\neq \tilde y}\,,
\end{align}
for an appropriate loss $\mathcal{L}_\theta$, with  $C_d$ being the optimal transport (OT) cost  associated with the cost function $c_d$ on $\mathcal{Z}\times\mathcal{Z}$, where $\mathcal{Z}\coloneqq \mathcal{X}\times\mathcal{X}\times\mathcal{Y}$. Specifically,
\begin{align}\label{eq:Cd_def_intro}
C_d(\mu,\nu)\coloneqq \inf_{\pi:\pi_1=\mu,\pi_2=\nu}\int  c_d(z,\tilde z)\pi(dzd\tilde z)\,,
\end{align}
where $\pi_1,\pi_2$ denote the marginals of $\pi\in\mathcal{P}(\mathcal{Z}\times\mathcal{Z})$ (the set of probability distributions on $\mathcal{Z}\times\mathcal{Z}$).
In \eqref{eq:OT_DRO_Bui}  the empirical distribution for a collection of sample/label pairs $(x_i,y_i)\in\mathcal{X}\times\mathcal{Y}$ is $P_n=\frac{1}{n}\sum_{i=1}^n \delta_{z_i}$, $z_i\coloneqq (x_i,x_i,y_i)$, and the adversarial distribution $Q$ is over the variable $\tilde z=(\tilde x,\tilde{x}^\prime,\tilde y)\in\mathcal{Z}$.  The duplication of the sample component in $z_i$ allows for both adversarial and original samples to be used in the loss; specifically, the cost in \eqref{eq:OT_DRO_Bui} forces $\tilde{x}^\prime$ to equal the original sample.  PGD (with a mixture of original and adversarial samples) \cite{madry2018towardsrobustoptimization}, TRADES,  \cite{zhang2019theoreticallytradeOff}, and MART \cite{wang2019improving_mart}, correspond to the following loss functions
\begin{align}\label{eq:PGD_T_M_losses}
    &\mathcal{L}_\theta^{PGD}(\tilde z)\coloneqq CE(h_\theta(\tilde{x}^\prime),\tilde y)+\beta CE(h_\theta(\tilde x),\tilde y)\,,\\
    &\mathcal{L}_\theta^{TRADES}(\tilde z)\coloneqq CE(h_\theta(\tilde{x}^\prime),\tilde y)+\beta {KL}(h_{\theta}(\tilde{x}^\prime),h_\theta(\tilde x))\,,\notag\\
    &\mathcal{L}_\theta^{MART}(\tilde z)\coloneqq BCE(h_\theta(\tilde{x}),\tilde y)+\beta(1-[h_\theta(\tilde{x}^\prime)]_{\tilde y})
    {KL}(h_\theta(\tilde{x}^\prime)\|h_\theta(\tilde x))\,,\notag
\end{align}
where $h_\theta$ is the classifier, $CE$ denotes cross-entropy, $KL$ is the KL divergence, and $BCE$ is defined in Eq.~(8) of \cite{wang2019improving_mart}. The UDR adversarial training methods proposed in \cite{bui_UDR_2022unified} generalized \eqref{eq:OT_DRO_Bui} by allowing for alternative OT-costs; specifically, they relaxed the hard $\epsilon$-ball constraint in $c_d$ to  a soft penalty \eqref{eq:c_UDR}  (see also the related approaches in \cite{sinha2018certifying} and \cite{regniez2021distributional}).

 In the present work we further generalize the DRO approach to adversarial robustness by proposing a novel class of divergences for comparing probability distributions that combines features of both OT costs and information-theoretic divergences (such as KL). We use these new optimal-transport-regularized divergences to  define  distribution neighborhoods for use in DRO \eqref{eq:DRO_gen}. This leads us to propose a novel class of  adversarial training methods that simultaneously transport adversarial samples (with general OT cost) and  re-weight the adversarial samples according to the information-theoretic divergence.  The former feature is shared with the UDR method \cite{bui_UDR_2022unified} but the ability of our approach to use information from the loss  together with the OT cost in order to  adversarially re-weight samples in a principled and dynamical manner during training is a qualitatively new feature of our method; this feature follows naturally from our more general DRO framework, which `mixes' information-theoretic and OT divergences via an infimal convolution (see Eq.~\ref{eq:Dc_def_intro} below).  In practice, the adversarial re-weighting causes the optimization algorithm to focus  on the samples in each minibatch that are more vulnerable to adversarial perturbation. Our approach complements and can be combined with methods that modify the OT cost (e.g., UDR) as well as methods which focus on modifying the loss (e.g., TRADES and MART); in the experiments in Section \ref{sec:experiments} we demonstrate the performance benefit of our framework when used in combination with UDR, TRADES, and MART.

{\bf Optimal-Transport-Regularized Divergences:}
The new divergences that we introduce in this work  are defined as an infimal convolution between an optimal transport cost, $C$, and an information divergence, $D$, e.g., a $f$-divergence, $D=D_f$ \cite{{LieseVajda}}, of which the KL-divergence is one example. More precisely,  given an OT cost function $c(z,\tilde z)$ and an information divergence, $D$, we define the {\bf OT-regularized divergence}, $D^c$, of a distribution $Q$ with respect to a distribution  $P$ by
\begin{flalign}\label{eq:Dc_def_intro}
D^c(Q\|P)\coloneqq\inf_{\eta\in\mathcal{P}(\mathcal{Z})}\{D(\eta\|P)+C(\eta,Q)\}\,,
\end{flalign}
where $\mathcal{P}(\mathcal{Z})$ denotes the set of probability distributions on the space $\mathcal{Z}$ and the optimal transport cost associated with the cost function $c$ is given by (see \cite{villani2008optimal})
\begin{align}\label{eq:C_def_intro}
C(\mu,\nu)\coloneqq \inf_{\pi:\pi_1=\mu,\pi_2=\nu}\int  c(z,\tilde z)\pi(dzd\tilde z)\,.
\end{align}
The only assumptions we make regarding $c$ are non-negativity, lower semicontinuity, and that $c(z,z)=0$ for all $z$.   Intuitively, one can view \eqref{eq:Dc_def_intro} as specifying a cost via a two-step procedure for transforming $P$ into $Q$.  First, one redistributes the probability-mass in $P$ to form an intermediate distribution $\eta$, paying the cost $D(\eta\|P)$ (we say redistribute because we focus on $D$ that are information divergences, meaning they are computable in terms of the likelihood ratio $d\eta/dP$, though most of our theorems in Section \ref{sec:proofs} apply more generally).  Second, one performs optimal transport to transform $\eta$ into $Q$, paying the cost $C(\eta,Q)$.  The optimal intermediate measure $\eta_*$ determines the final cost $D^c(Q\|P)$. The infimal convolution structure, including the bound $D^c(Q\|P)\leq \min\{D(Q\|P),C(P,Q)\}$, causes $D^c$ to inherit properties from both $D$ and $C$ and allows it to interpolate between these two extremes; see Section \ref{sec:properties}. The OT-regularized divergences are related to the $\Gamma$-divergences defined in \cite{Dupuis:Mao}, $(f,\Gamma)$-divergences defined in \cite{JMLR:v23:21-0100}, and the IC-$\Gamma$-R{\'e}nyi divergences from \cite{birrell2023functionspace}, but here we utilize optimal transport costs as opposed to integral-probability-metric (IPM) regularization of information divergences.  We found that OT-regularization is more naturally suited to adversarial robustness methods than IPM regularization from a mathematical perspective; { specifically, the definition of OT-cost via a minimization over couplings \eqref{eq:C_def_intro} pairs well with the structure of the computation in \eqref{eq:DRO_initial}\,-\,\eqref{eq:DRO_identity_formal} below, often leading to an equivalent low dimensional optimization as in \eqref{eq:Df_Gibbs} and \eqref{eq:Renyi_cc}. In contrast, the definition of an IPM in terms of maximizing over a space of test functions does not naturally lead to  such relatively simple and computationally   tractable reformulations. } In addition,  those prior works on IPM regularization focused on the  equality of the primal and dual formulas for the divergence, which facilitates applications to GANs; here we focus on adversarial robustness, which requires different techniques.

{  Alternative approaches to regularizing  OT have  previously appeared in the DRO literature, such as  Sinkhorn-DRO \cite{wang2021sinkhorn}, which uses the well-known Sinkhorn divergence (entropic regularization) \cite{cuturi2013sinkhorn}, and  \cite{azizian2023regularization} which considers more general  penalties (e.g., generalizing Sinkhorn to use other $f$-divergence penalties).  While at a high-level our approach shares some features with these convex-penalty regularization methods, our framework is distinguished  by the use of an infimal  convolution of two different divergences \eqref{eq:Dc_def_intro}.  This has a very different operational meaning than entropic (or more general $f$-divergence penalty) regularization, which penalizes the deviation of the transport plan from a baseline coupling (e.g., a product).  In contrast, our approach has the intuitive interpretation of an  optimal two stage transformation process that is jointly controlled by OT and a second divergence; we argue that this is more appropriate for applications such as adversarial training, wherein it corresponds to an optimal combination of adversarially reweighting and transporting of the training  samples.}

We use the OT-regularized divergences \eqref{eq:Dc_def_intro} to define distribution neighborhoods of size $\epsilon>0$; this leads to  the following OT-regularized DRO problem, which we will employ as a tool for enhancing adversarial robustness:
\begin{align}\label{eq:DRO_Dc}
\inf_\theta\sup_{Q: D^c(Q\|P_n)\leq \epsilon}E_Q[\mathcal{L}_\theta]\,.
\end{align}
The OT-regularized-divergence neighborhoods are qualitatively different from both $f$-divergence and optimal-transport neighborhoods, as they allow for a combination of probability-mass transport and redistribution when forming the perturbed distributions, $Q$.  This allows for the support of $Q$ to differ from that of $P_n$ (as in \cite{bui_UDR_2022unified}) and also for the probability of widely separated modes to be re-weighted, something that is not possible with pure optimal-transport neighborhoods.  

When viewed as tools for adversarial training, we call \eqref{eq:DRO_Dc} the  $ARMOR_D$ methods, standing for \textbf{A}dversarially \textbf{R}obust \textbf{M}odels with \textbf{O}ptimal-Transport-\textbf{R}egularized \textbf{D}ivergences.  We note that one retains the freedom to choose the loss function in \eqref{eq:DRO_Dc}, with the options in \eqref{eq:PGD_T_M_losses} leading to the $ARMOR_D-PGD$, $ARMOR_D-TRADES$, and $ARMOR_D-MART$ methods. Furthermore, by an appropriate choice of OT cost, one obtains $ARMOR_D-UDR$ variants of these methods. 

In Section \ref{sec:formal_derivation}, we show that, for a general choice of loss function and OT cost and for appropriate choices of $D$, the DRO problem \eqref{eq:DRO_Dc} can be converted into a computationally tractable optimization problem; in particular, we find that the inclusion of the information divergence $D$ adds only slightly to the computational cost, as compared to pure OT-based methods. In Section \ref{sec:reweighting} we provide a formal solution in the case $D=D_f$, thereby clarifying the manner in which our method combines optimal transport and adversarial re-weighting.    In Section \ref{sec:properties} we  list a number of properties of the OT-regularized divergences, thus demonstrating that they are well-behaved mathematical objects; precise statements and proofs are found in Section \ref{sec:proofs}. In Section \ref{sec:experiments} we empirically evaluate the $ARMOR_D$ methods on  CIFAR-10 and CIFAR-100 image classification, where we find it offers significant performance gains. 

{ In sum, we find the proposed  $ARMOR_D$ methods to be promising tools for addressing the persistent and pressing threat that adversarial attacks  present to real-world AI systems \cite{survey,adversarial_IEEEaccess,carlini_1}.  The $ARMOR_D$ framework is sufficiently flexible to be layered onto other methods, as demonstrated in the examples in Section \ref{sec:experiments}, with little added computational cost. Moreover, our general DRO framework \eqref{eq:DRO_Dc} can be applied to any  stochastic optimization problem,  including (but not limited to) the deep-learning classification problems considered below.}

\subsection{Comparison with Related DRO Frameworks}\label{sec:related_approaches}
{ 
Two related, but distinct, DRO approaches to ours, both also utilizing a combination of optimal transport and reweighting, were released in close proximity to the present work. He we provide a detailed comparison between them and our OT-regularized-divergence framework.

1) The smoothed $f$-divergence method proposed in \cite{liu2023smoothed} is defined in terms of minimizing an $f$-divergence over an OT-cost ball. They primarily focus on   deriving statistical bounds on the out-of-sample performance.  As they show in their Theorem 6.3, their smoothed $f$-divergences are closely related to infimal convolutions. However, our framework is significantly more general, in  that it allows for OT-regularization of more than $f$-divergences, e.g.,  R{\'e}nyi divergences \eqref{eq:Renyi_cc};  moreover, our Theorem \ref{thm:CC_general_D}  provides a concrete tool for the construction   of a wide range of   divergences to which   our OT-regularized divergence theory applies.  

2)  The conditional moment constraint approach of \cite{2023arXiv230805414B}  generalizes standard OT by introducing a cost function on $\mathcal{Z}\times\mathbb{R}^+$ that simultaneously allows for probability-mass transport and redistribution. In the special case of $D=D_f$,  our approach agrees with the use of a particular form of cost function in  the method of \cite{2023arXiv230805414B}, with both  reducing to the same finite dimensional optimization problem under appropriate assumptions; see their Theorems 4.1, 5.1, and Proposition 5.1 and compare with  our Theorem \ref{thm:DRO_gen_P_gen_D} and Eq.~\eqref{eq:OT_reg_Df_DRO}-\eqref{eq:OT_reg_KL_DRO}.  However, both our approach and that of \cite{2023arXiv230805414B} allow for a great deal of freedom in constructing new divergences, with neither clearly subsuming the other at a   general level.   In particular, note that \eqref{eq:DH_DRO}, which is the combination of our Theorems \ref{thm:DRO_gen_P_gen_D} and \ref{thm:CC_general_D}, allows for divergences to be constructed from functionals $H_P$ that can depend nonlinearly on the expectations under $P$, as in R{\'e}nyi divergences \eqref{eq:Renyi_cc} or variance penalties.  It is not clear that such terms can be incorporated into the approach of \cite{2023arXiv230805414B}; see  Eq.(D) above their Theorem 4.1. 

In addition, both \cite{liu2023smoothed} and \cite{2023arXiv230805414B} assume the probability distributions under consideration have compact support, whereas much of our theory, including the tractable DRO reformulation results in Theorem \ref{thm:DRO_gen_P_gen_D} and Corollary \ref{cor:OT_f_div_DRO}, apply to distributions with non-compact support. Our work is also distinguished   through the proofs of a number of  properties of the OT-regularized divergences that do not have analogues in \cite{liu2023smoothed,2023arXiv230805414B} (see the summary in Section \ref{sec:properties}), and through our novel use of \eqref{eq:OT_reg_Df_DRO}, \eqref{eq:OT_reg_KL_DRO}, and \eqref{eq:OT_reg_Renyi_DRO} as tools for enhancing adversarial robustness,  where  we find substantial performance gains with only a slight increase in computational cost. }

\section{OT-Regularized Divergences: DRO Identity and Properties}\label{sec:formal_derivation}
In general, the DRO problem \eqref{eq:DRO_Dc} is an intractable infinite dimensional optimization problem. However, for appropriate choices of $D$ one can derive a finite dimensional reformulation that leads to computationally efficient implementations.  In this section we provide a formal derivation of the key identity.  For  the required assumptions and a rigorous proof, see Section \ref{app:DRO_identity_proofs}.  

Noting that $C$ is jointly convex and assuming that $D$ is convex in its first argument (as is the case when $D$ is a $f$-divergence) one can see from \eqref{eq:Dc_def_intro} that $D^c$ is convex in its first argument. Therefore the DRO problem is a convex optimization problem and one can compute 
\begin{align}
&\sup_{Q: D^c(Q\|P_n)\leq \epsilon}E_Q[\mathcal{L}_\theta]\label{eq:DRO_initial}\\
=&\inf_{\lambda>0}\{\lambda \epsilon+\sup_{Q\in\mathcal{P}(\mathcal{Z})}\{E_Q[\mathcal{L}_\theta]-\lambda D^c(Q\|P_n)\}\}\label{eq:convex_duality}\\
=&\inf_{\lambda>0}\{\lambda \epsilon+\sup_{Q,\eta\in\mathcal{P}(\mathcal{Z})}\{E_Q[\mathcal{L}_\theta]-\lambda D(\eta\|P_n)-\lambda C(\eta,Q)\}\}\label{eq:DRO_identity_Dc_def}\\
=&\inf_{\lambda>0}\{\lambda \epsilon+\sup_{\eta\in\mathcal{P}(\mathcal{Z})}\{ -\lambda D(\eta\|P_n)+ \sup_{Q\in \mathcal{P}(\mathcal{Z})}\sup_{\pi:\pi_1=\eta,\pi_2=Q}\{E_Q[\mathcal{L}_\theta]-\lambda E_\pi[c]\}\}\}\label{eq:DRO_identity_C_def}\\
=&\inf_{\lambda>0}\left\{\lambda \epsilon+\lambda\sup_{\eta\in\mathcal{P}(\mathcal{Z})}\left\{ - D(\eta\|P_n)+ \sup_{\pi_z(d\tilde z)}\int \int \mathcal{L}_\theta (\tilde z)/\lambda -  c(z,\tilde z)\pi_z(d\tilde z)\eta(dz)\right\}\right\}\label{eq:DRO_identity_prob_kernel}\\
=&\inf_{\lambda>0}\left\{\epsilon\lambda +\lambda\sup_{\eta\in\mathcal{P}(\mathcal{Z})}\left\{  \int\lambda^{-1} \sup_{\tilde z\in\mathcal{Z}}\{ \mathcal{L}_\theta (\tilde z) -  \lambda c(z,\tilde z)\}\eta(dz)- D(\eta\|P_n)\right\}\right\}\,.\label{eq:DRO_identity_formal}
\end{align}
The equality \eqref{eq:convex_duality} is obtained using strong duality, lines \eqref{eq:DRO_identity_Dc_def} and \eqref{eq:DRO_identity_C_def} are obtained using the definitions \eqref{eq:Dc_def_intro} and \eqref{eq:C_def_intro} of $D^c$ and $C$ along with properties of suprema and infima, \eqref{eq:DRO_identity_prob_kernel} recognizes that the suprema over $Q$ and $\pi$ can be rewritten as a supremum over probability kernels $\pi_z(d\tilde z)$, and finally   \eqref{eq:DRO_identity_formal} uses the fact that  the supremum over probability kernels  achieves the pointwise supremum of the integrand. To this point, the derivation closely follows that of \cite{EsfahaniKuhn2018} for Wasserstein DRO, as well as the adversarial robustness approach by \cite{bui_UDR_2022unified}. Note that  effect of the OT cost is  to replace the loss $\mathcal{L}_\theta$ with what we call the OT-regularized loss
\begin{align}\label{eq:OT_regularized_loss_main}
\mathcal{L}^{c}_{\theta,\lambda}(z)\coloneqq  \sup_{\tilde z\in\mathcal{Z}}\{ \mathcal{L}_\theta (\tilde z) -  \lambda c(z,\tilde z)\}\,, 
\end{align}
which is known as the $c$-transform in the optimal transport literature; see Definition  5.2 in \cite{villani2008optimal}. The importance of the $c$-transformed loss for Wasserstein DRO is well known; see the references to prior work on Wasserstein and OT-DRO in the introduction. The supremum over $\tilde z\in\mathcal{Z}$ in \eqref{eq:OT_regularized_loss_main} can be thought of as selecting an adversarial sample that is paired with each real sample, $z$. We note that our mathematical framework can be used to robustify any empirical risk minimization problem, not only classification, and so our notation here does not  explicitly decompose the variables into sample and label components.

\subsection{Convex Conjugate of $D$}
Comparing our OT-regularized-divergence DRO framework with Wasserstein and OT-DRO, the new ingredient  is the optimization over $\eta$ in \eqref{eq:DRO_identity_formal}. This can be recognized as  the convex-conjugate of $\eta\mapsto D(\eta\|P_n)$ and for certain choices of $D$ this term can be reformulated as a finite dimensional convex optimization problem.  In particular, here we present explicit results for two  important families of divergences, namely the $f$-divergences and R{\'e}nyi divergences. { 
Section \ref{sec:CC_general_D} below presents a more general theory that can be applied to other convex functionals.}
\begin{enumerate}
    \item {\bf $f$-Divergences:} Using the generalization of the Gibbs variational principle to $f$-divergences, see Theorem 4.2 in \cite{BenTal2007}, one has
\begin{align}\label{eq:Df_Gibbs}
\sup_{\eta\in\mathcal{P}(\mathcal{Z})} \{E_\eta[\phi]-D_f(\eta\|P)\}=\inf_{\rho\in\mathbb{R}}\{\rho+ E_P[f^*(\phi-\rho)]\}\,,
\end{align}
where $f^*$ is the Legendre transform of $f$.  
\item { {\bf R{\'e}nyi Divergences:} Building off of the variational representation of R{\'e}nyi divergences that was derived in \cite{birrell2023functionspace}, one can compute the convex conjugate of the R{\'e}nyi divergences:
\begin{align}\label{eq:Renyi_cc}
&\sup_{\eta\in \mathcal{P}(\mathcal{Z})}\{E_\eta[\phi]-R_\alpha(P\|\eta)\}
= \inf_{\rho\in\mathbb{R}}\{\rho+\Lambda_\alpha^P[\phi-\rho]\}\,,\\
&\Lambda_\alpha^P[g]\coloneqq-\left(\frac{1}{\alpha-1}\log E_P\left[ |g|^{(\alpha-1)/\alpha} \right]+\alpha^{-1}(\log\alpha +1)\right)1_{g<0}+\infty1_{g\not<0}\,.\notag
\end{align}
To the best of the authors' knowledge, the above result is new; see Theorem \ref{thm:Renyi_cc} in Appendix \ref{app:additional_proofs} for details. Note the  reversal of the order of arguments in $R_\alpha(P\|\eta)$ as compared to $D_f(\eta\|P)$ in \eqref{eq:Df_Gibbs}.  This is due to the R{\'e}nyi divergences being convex only in their second argument when $\alpha>1$.
}
\end{enumerate}
{ In the numerical experiments in Section \ref{sec:experiments} we primarily focus on the case where $D$ is a $f$-divergence, with some additional exploration of R{\'e}nyi divergences. Beyond this, we emphasize that our general framework  opens the door to   further exploration and experimentation with a broad range of alternative convex functionals; see Theorem \ref{thm:CC_general_D} for a theoretical tool that enables such experimentation.

Choosing $D=D_f$, by combining \eqref{eq:DRO_initial}\,-\,\eqref{eq:DRO_identity_formal} with \eqref{eq:Df_Gibbs} } we obtain the following finite-dimensional reformulation of the OT-regularized DRO problem
\begin{align}\label{eq:OT_reg_Df_DRO}
\inf_{\theta\in\Theta}\sup_{Q: D_f^c(Q\|P_n)\leq \epsilon}E_Q[\mathcal{L}_\theta]=\inf_{\lambda>0,\rho\in\mathbb{R},\theta\in\Theta}\left\{ \epsilon\lambda+ \rho+\lambda \frac{1}{n}\sum_{i=1}^nf^*(\lambda^{-1}(\mathcal{L}_{\theta,\lambda}^{c}(z_i)-\rho))\right\}\,.
\end{align}
Here we made the  change of variables $\rho\to \rho/\lambda$ so that  the objective function is jointly convex in $\lambda,\rho$ (see Corollary \ref{cor:OT_f_div_DRO}). Note that the new variables $\lambda,\rho$ simply augment the minimization over model  parameters $\theta$ by adding two real variables, which adds a relatively small additional computational cost when applied to deep learning (due to the dimension of $\theta$ being large). The $\lambda$ parameter has the same interpretation as in the  OT-DRO based method \cite{bui_UDR_2022unified}; it can be viewed as a dynamical OT-cost weight, selected according to the optimization \eqref{eq:DRO_identity_formal}, which is tied to the neighborhood size $\epsilon$.  This perspective is most apparent in \eqref{eq:convex_duality}. The significance of $\rho$ will be discussed in Section \ref{sec:reweighting} below. In Section \ref{sec:experiments} we will experiment with the KL divergence and the family of $\alpha$-divergences (i.e., $f=f_\alpha$ as in Eq.~\ref{eq:f_alpha_def}), which we call the $ARMOR_{KL}$ and $ARMOR_\alpha$ methods respectively. An explicit formula for $f^*$ in the case of $\alpha$-divergences is given in \eqref{eq:f_alpha_star}. In the KL-divergence case the minimization over $\rho$ can be evaluated analytically, yielding
\begin{align}\label{eq:OT_reg_KL_DRO}
\inf_{\theta\in\Theta}\sup_{Q: KL^c(Q\|P_n)\leq \epsilon}E_Q[\mathcal{L}_\theta]=\inf_{\lambda>0,\theta\in\Theta}\left\{ \epsilon\lambda+\lambda \log\left(\frac{1}{n}\sum_{i=1}^n\exp(\lambda^{-1}\mathcal{L}_{\theta,\lambda}^{c}(z_i))\right)\right\}\,.
\end{align}
{ Similarly, choosing $D(Q\|P)=R_\alpha(P\|Q)$ leads to the $ARMOR_{R{\'e}n}$ methods
\begin{align}\label{eq:OT_reg_Renyi_DRO}
\inf_{\theta\in\Theta}\sup_{Q:R^c_\alpha(P_n\|Q)\leq \epsilon}E_Q[\mathcal{L}_\theta]=\inf_{\lambda>0,\rho\in\mathbb{R},\theta\in\Theta}\left\{\epsilon\lambda+\rho+\lambda \Lambda_\alpha^{P_n}\left[\lambda^{-1}(\mathcal{L}^c_{\theta,\lambda}-\rho)\right]\right\}\,.
\end{align}
We will refer to any of \eqref{eq:OT_reg_Df_DRO}, \eqref{eq:OT_reg_KL_DRO}, or \eqref{eq:OT_reg_Renyi_DRO}} as the outer minimization problem and will call \eqref{eq:OT_regularized_loss_main}  the inner maximization problem.

\subsection{Interpreting the Outer Minimizer: Adversarial Sample Weights}\label{sec:reweighting}

In this section we give an intuitive interpretation of the minimization over the auxiliary parameters $\lambda,\rho$ in the $f$-divergence case  \eqref{eq:OT_reg_Df_DRO}; they can  be viewed as the computation of optimal adversarial weights for the adversarial samples, where optimality is defined in part by the chosen $f$-divergence. This is a complement to the inner maximizer \eqref{eq:OT_regularized_loss_main} which constructs the optimally transported adversarial samples, according to the chosen OT cost function. This interpretation gives insight into the qualitatively novel nature of our method.

Letting $\tilde z_i(\lambda)$ be the solution to the inner maximizer \eqref{eq:OT_regularized_loss_main} with $z=z_i$ and $\lambda_*$, $\rho_*$ be the optimal scaling and shift parameters  for the outer minimizer at a fixed $\theta$ (we suppress the $\theta$-dependence of $\tilde z_i$, $\lambda_*$, and $\rho_*$ in the notation) we formally derive the following reformulation of \eqref{eq:OT_reg_Df_DRO} in Appendix \ref{app:reweighting}:
\begin{align}\label{eq:OT_reg_Df_DRO_optim_main}
\inf_{\lambda>0,\rho\in\mathbb{R}}\left\{ \epsilon\lambda+ \rho+\lambda \frac{1}{n}\sum_{i=1}^nf^*(\lambda^{-1}(\mathcal{L}_{\theta,\lambda}^{c}(z_i)-\rho))\right\}=E_{Q_{*,\theta}}[\mathcal{L}_\theta]\,,
\end{align}
where the optimal adversarial distribution is $Q_{*,\theta}\coloneqq\sum_{i=1}^n p_{*,i} \delta_{\tilde z_i(\lambda_*)}$, 
having optimal adversarial weights
\begin{align}\label{eq:new_weights_main}
   p_{*,i}\coloneqq\frac{1}{n}(f^*)^\prime((\mathcal{L}^c_{\theta,\lambda_*}(z_i)-\rho_*)/\lambda_*)\,.
\end{align}
This shows that the minimization over $\theta$ in \eqref{eq:OT_reg_Df_DRO} solves the risk minimization problem for the ($\theta$-dependent) optimal adversarial distribution $Q_{*,\theta}$. The optimal adversarial distribution is supported on the optimal adversarial samples  $\tilde z_i(\lambda_*)$ and the weight of the $i$'th sample is changed from $1/n$ to $p_{*,i}$ \eqref{eq:new_weights_main}.  To understand the significance of the re-weighting  $p_{*,i}$, first recall that $f^*$ is non-decreasing (see Definition \ref{def:f_div} and Corollary \ref{cor:OT_f_div_DRO}), hence $p_{*,i}\geq 0$. In  \eqref{eq:new_weight_sum} we show that the $p_{*,i}$'s sum to $1$. Convexity of $f^*$ implies that $(f^*)^\prime$ is also non-decreasing, hence the $p_{*,i}$'s shift more weight towards the samples where the OT-regularized loss is larger, as would be expected for an adversarial re-weighting.  In some cases, such as for the $\alpha$-divergences, $f^*$ is constant on $(-\infty,M)$ for some $M$ ($M=0$ when $f=f_\alpha$, as seen in Eq.~\ref{eq:f_alpha_star}).  In such cases, samples with $\mathcal{L}^c_{\theta,\lambda_*}(z_i)<\rho_*+M\lambda_*$  have their weighting changed to $0$. Intuitively, one can consider those samples as having sufficiently small OT-regularized loss and hence the method moves its attention away from them to focus on the more troublesome samples. Note that samples are only temporarily ignored; attention may return to them later in the training if their loss moves above the (dynamic) threshold.  Part of the task of the outer minimizer is to dynamically determine the optimal threshold for `sufficient smallness', as set by $\rho_*+M\lambda_*$. We note that this threshold changes with $\theta$, as  $\lambda_*$ and $\rho_*$ are both $\theta$-dependent. The ability of the  $ARMOR_D$ methods to re-weight adversarial samples in addition to transporting them is the primary innovation of our approach, as compared to the prior OT-DRO based robustness method \cite{bui_UDR_2022unified} or the earlier soft-constraint based method \cite{sinha2018certifying}.  As we demonstrate in the examples in Section \ref{sec:experiments}, this is a powerful new ingredient  and is made possible because our DRO neighborhoods incorporate both information-theoretic and OT components via the infimal convolution  \eqref{eq:Dc_def_intro}. Our approach is distinct from the re-weighting method proposed in \cite{guo2022learning} for addressing the problem of class imbalance in the training data, which is not an adversarial re-weighting. Our method is also distinct from the approach in \cite{zhang2020geometry} where modified weights were introduced manually, based on an informal notion of distance to the decision boundary. In contrast,  re-weighting in $ARMOR_D$ is determined in a principled manner by the DRO framework, via the choice of $f$ and $c$; it uses information from the OT-regularized loss of each sample during training, along with a dynamic threshold, as seen in \eqref{eq:new_weights_main}.  In particular the adaptive threshold, which determines which samples the optimizer currently considers `troublesome', is a qualitatively novel feature of our method.  



\subsection{Properties of the OT-Regularized Divergences}\label{sec:properties} The OT-regularized divergences have many attractive mathematical properties, making them well suited to DRO as well as other statistical learning tasks. We summarize a number of these properties here;  see Section \ref{sec:proofs} for precise statements of the required assumptions along with proofs. Given appropriate assumptions on $D$ and $c$ one has the following:
\begin{enumerate}
\item $D^c(\nu\|\mu)\geq 0$ and $D^c(\nu\|\mu)=0$ if and only if $\nu=\mu$; see Theorem \ref{thm:D_c_div_property}. This divergence property implies that $D^c(\nu\|\mu)$ can be interpreted as measuring the discrepancy between $\nu$ and $\mu$.
\item There exists an optimal intermediate distribution that solves the minimization problem in the definition \eqref{eq:Dc_def_intro}. More specifically, there exists $\eta_*$ such that
\begin{align}
D^c(\nu\|\mu)=D(\eta_*\|\mu)+C(\eta_*,\nu)
\end{align}
and this $\eta_*$ is unique under appropriate assumptions; see Theorem \ref{thm:inf_conv_solution}.
\item $D^c(\nu\|\mu)$ is convex in $\nu$ (see Lemma \ref{lemma:D_c_convex_property}).  This implies that the DRO neighborhoods $\{Q:D^c(Q\|P_n)\leq \epsilon\}$ are convex sets and is also key in the derivation of the DRO identity \eqref{eq:DRO_identity_formal}. 
\item $D^c(\nu\|\mu)$ is lower semicontinuous in $\nu$ (see Theorem \ref{thm:Dc_LSC}). This property is useful for theoretical purposes and it implies that the DRO neighborhoods $\{Q:D^c(Q\|P_n)\leq \epsilon\}$ are closed sets.
\item $D^c$ interpolates between $D$ and $C$ in the following sense:  For $r>0$ define the scaled cost function $c_r=rc$.  Then
\begin{align}
&\lim_{r\to0^+}r^{-1}D^{c_r}(\nu\|\mu)=C(\mu,\nu) \,\,\,\, \text{(see Theorem \ref{thm:limit_D_c_to_C})}\,,\label{eq:Dc_to_C_preview}\\
&\lim_{r\to \infty}D^{c_r}(\nu\|\mu)=D(\nu\|\mu) \,\,\,\, \text{(see Theorem \ref{thm:limit_D_c_to_D}).}\label{eq:Dc_to_D_preview}
\end{align}
Informally, this property implies that DRO over both $D$ and $C$ neighborhoods can be viewed as special cases of DRO over $D^c$ neighborhoods. More specifically, \eqref{eq:Dc_to_C_preview} indicates that when $r$ is sufficiently small, DRO over the neighborhood $\{Q:D^{c_r}(Q\|P_n)\leq r\epsilon\}$ is approximately the same as DRO over the neighborhood $\{Q:C(P_n,Q)\leq \epsilon\}$. Similarly, \eqref{eq:Dc_to_D_preview} indicates that when $r$ is sufficiently large, DRO over the neighborhood $\{Q:D^{c_r}(Q\|P_n)\leq \epsilon\}$ is approximately the same as DRO over the neighborhood $\{Q:D(Q\|P_n)\leq \epsilon\}$ (see Theorems \ref{thm:DRO_limit_r_0} and \ref{thm:DRO_limit_r_infty} for precise statements).  Therefore if one includes the scale factor $r$ and neighborhood size $\epsilon$ as tunable hyperparameters (as we do in the experiments in Section \ref{sec:experiments}) then the special cases of $C$ and $D$ neighborhoods will be (approximately) explored in the process of tuning an $ARMOR_D$ method.
\end{enumerate}
We note that these properties   do not require the distributions to have compact support, except for the DRO interpolation results in  Theorems \ref{thm:DRO_limit_r_0} and \ref{thm:DRO_limit_r_infty}.

\section{Experimental Results}\label{sec:experiments}
We empirically evaluate the $ARMOR_D$ adversarial robustness methods using several widely-used deep-learning-based image classification benchmark datasets: MNIST, CIFAR-10, and CIFAR-100.  In the  preliminary evaluation on MNIST, we compare the performance of different variants of $ARMOR_D$, including the use of different information divergences. Then on CIFAR, we demonstrate the ability of $ARMOR_D$ to augment various robust losses used in the literature and show that it leads to substantial performance improvements, thereby illustrating the effectiveness of our adversarial re-weighting mechanism in $ARMOR_D$. The experimental setup and the threat model in our experiments follow \cite{bui_UDR_2022unified} and \cite{carlini2017towards} as detailed in Section~\ref{app:setup} and Section~\ref{app:threat_model}. 

\subsection{Preliminary Evaluation of $ARMOR_D$ Variants}
\label{sec:mnist}
To gain initial insight regarding the various $ARMOR_D$ variants, before we conduct comparisons with counterparts on large datasets, we present the results of evaluating  $ARMOR_D$ on the standard MNIST dataset for various $D$'s and with or without natural samples. The results are summarized in Table~\ref{tab:mnist}.

\begin{table}[H]
\small
\centering
\setlength\tabcolsep{1.5pt}
\caption{{\bf Preliminary Evaluation of $ARMOR_D$ Variants on MNIST} Comparison of the performance of our proposed method for enhancing the robustness on the MNIST dataset. 
$adv_s$  denotes the use of adversarial samples constructed via \eqref{eq:OT_regularized_loss_main} and $nat$ refers to the use of natural samples alongside the adversarial samples, as described in Section \ref{app:adv+nat_description}. }
\label{tab:mnist}
\begin{tabular}{|l|l|l|l||l|l|l||l|l|l|}
\hline
& \multicolumn{3}{c|}{PGD Attack} & \multicolumn{3}{c|}{FGSM Attack} & \multicolumn{3}{c|}{No Attack}\\\hline

\textbf{Defense} & Acc & FNR & FPR & Acc & FNR & FPR & Acc & FNR & FPR\\ \hline

Non-robust & 25.36\% & 74.23\% & 8.29\% & 52.76\%& 46.96\%& 5.24\% & 99.01\% & 1.00\% & 0.11\%\\ \hline\hline

FGSM & 96.63\% & 3.39\% & 0.37\%& 97.45\%& 2.57\% & 0.28\%& 99.05\%& 0.96\%& 0.11\%\\ \hline

PGD & 97.10\% & 2.92\% &0.32\% & 97.26\% & 2.76\% & 0.30\%& 99.15\%& 0.85\%& 0.09\%\\ \hline





$ARMOR_{R{\'e}n}$ ($adv_{s}$) & 97.46\% & 2.56\%& 0.28\%& 97.64\% & 2.38\% & 0.26\%&98.99\% &1.02\%  & 0.11\% \\ \hline

$ARMOR_{R{\'e}n}$ ($adv_{s}+nat$) & 97.32\% & 2.71\%& 0.29\% & 97.55\% & 2.48\% & 0.27\% & 98.87\%& 1.15\% & 0.12\%\\ \hline

$ARMOR_{\alpha}$ ($adv_s$) & \underline{97.31\%} & \underline{2.71\%} & \underline{0.30\%} & {97.63\%} & {2.40\%} & {0.26\%}& 99.15\%& 0.86\%& 0.09\% \\ \hline

$ARMOR_{\alpha}$ ($adv_s+nat$) & \underline{98.05\%} & \underline{1.97\%} & \underline{0.22\%} & \underline{\textbf{98.25\%}} & \underline{\textbf{1.77\%}} & \uline{\textbf{0.19\%}}& \uline{\textbf{99.30\%}} & \uline{\textbf{0.70\%}} & \uline{\textbf{0.08\%}} \\ \hline

$ARMOR_{KL}$ ($adv_s$) & \underline{97.68\%}  & \underline{2.34\%} & \underline{0.26\%} & \underline{98.10\%} & \underline{1.92\%} & \underline{0.21\%} & {99.25\%} & {0.76\%}& {0.08\%} \\ \hline

$ARMOR_{KL}$ ($adv_s+nat$) & \underline{\textbf{98.11\%}}  & \underline{\textbf{1.91\%}} & \underline{\textbf{0.21\%}} & \underline{98.23\%} & \underline{1.78\%} & \underline{0.20\%}& 99.09\%& 0.92\%& 0.10\% \\ \hline

\end{tabular}
\begin{tablenotes}
      \small
      \item \textit{Note:} Best results are shown in bold font.  The results for methods that outperform the non-robust model and prior adversarial robustness methods across all three metrics are underlined.
\end{tablenotes}
\end{table}

Specifically, we consider the cases where $D$ is a R{\'e}nyi divergence, $\alpha$-divergence, or the KL divergence. We observe that all variants of $ARMOR_D$  tested here lead to performance improvements when under attack, as compared to the FGSM and PGD defenses. The best performance is obtained by $ARMOR_D$ where $D$ is an $\alpha$-divergence or the KL divergence, and using a mixture of adversarial and natural samples; we focus on these cases in the experiments presented in Section \ref{sec:CIFAR10} below.

\subsubsection{Qualitative Analysis of $ARMOR_D$}
\label{mnist_further}

To qualitatively verify the effectiveness of the proposed method, we examine a subset of MNIST test digits for which the performance of the FGSM, PGD, and our $ARMOR_{\alpha}$  methods differ, i.e., the digits for which at least one, but not all three of these methods failed under adversarial attack. This subset consists of 225 digits out of the 10,000 digit MNIST test set. The models robustified with FGSM, PGD, and our method correctly classified 54 digits (24.00\%), 100 digits (44.44\%), and 198 digits (88.00\%) out of this set, respectively. Figure~\ref{fig:mnist} shows the results on 12 randomly selected examples from the constructed set of 225 digits. We see that our method, $ARMOR_{\alpha}$ ($adv_{s}+nat$) shown in Figure~\ref{fig:ours}, exhibits significantly greater robustness on this subset of digits, though there remains a subset of ``difficult'' samples on which all three methods fail under adversarial attack. 


\begin{figure}[!h]
     \captionsetup[subfigure]{aboveskip=-1pt, belowskip=-1pt, justification=centering}
     \begin{subfigure}[b]{0.48\textwidth}
         \centering
         \includegraphics[width=\textwidth]{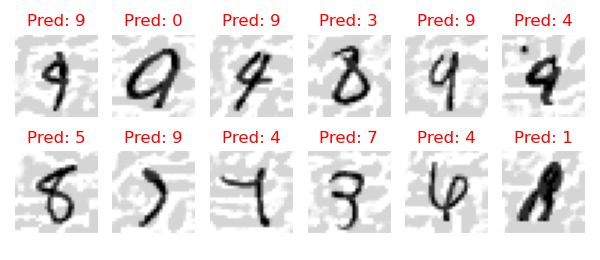}
         \caption{Non-robustified CNN model under the $PGD^{40}$ attack\newline}
         \label{fig:plain}
     \end{subfigure}
     \hfill
     \begin{subfigure}[b]{0.48\textwidth}
         \centering
         \includegraphics[width=\textwidth]{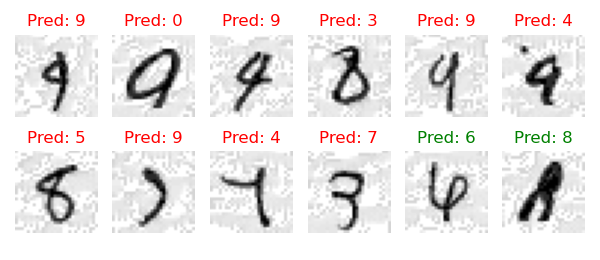}
         \caption{Robustified CNN model using adversarial training with FGSM under the $PGD^{40}$ attack}
         \label{fig:fgsm}
     \end{subfigure}
     \\
     \begin{subfigure}[b]{0.48\textwidth}
         \centering
         \includegraphics[width=\textwidth]{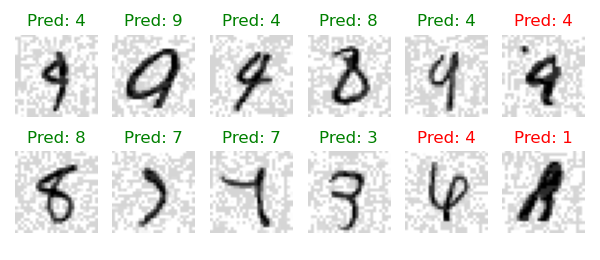}
         \caption{Adversarial training with $PGD^{20}$ under the $PGD^{40}$ attack}
         \label{fig:pgd}
     \end{subfigure}
     \hfill
     \begin{subfigure}[b]{0.48\textwidth}
         \centering
         \includegraphics[width=\textwidth]{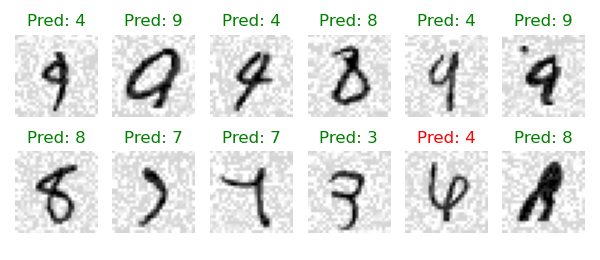}
         \caption{Adversarial training with our $ARMOR_{\alpha}$ under the $PGD^{40}$ attack}
         \label{fig:ours}
     \end{subfigure}
        \caption{The majority of the images, all except the second digit in the bottom row (``7''), were classified correctly by the non-robust CNN model when not under attack, but all 12 samples modified with the $PGD^{40}$ attack successfully mislead the non-robust CNN (Figure~\ref{fig:plain}). The robustified model via adversarial training with FGSM leads to two correct predictions (for digits ``6'' and ``8'') (Figure~\ref{fig:fgsm}). The number of correct predictions in this sample increases to 9 via adversarial training with $PGD^{20}$ (Figure~\ref{fig:pgd}). The number of correct predictions increases to 11 with our proposed method under the same attack (Figure~\ref{fig:ours}). }
        \label{fig:mnist}
\end{figure}

\subsection{Benchmark Comparisons}\label{sec:CIFAR10}

In this experiment we demonstrate the flexibility of  $ARMOR_D$ as a tool for augmenting previous approaches to adversarially robust learning and show the performance benefits  conferred by its novel features; specifically, we will compare UDR, TRADES, and MART with their $ARMOR_D$ variants.

{\bf Benchmark Methods and Evaluation Metrics:} Following \cite{bui_UDR_2022unified}, we evaluate the methods against the Projected Gradient Method ($PGD^{200}$) attack and the much stronger and more recent  AutoAttack  \cite{croce2020reliable}.  In these experiments we are evaluating the effect of  $ARMOR_D$'s adversarial re-weighting mechanism  and therefore we leave the OT cost unchanged (i.e., $c_d$ in \eqref{eq:OT_DRO_Bui} for TRADES and MART and the soft penalty \eqref{eq:c_UDR} for UDR). Our primary comparison will be the recent OT-DRO based method  \cite{bui_UDR_2022unified}, UDR, which like our method can also be used to enhance \eqref{eq:classical_robust_optimization} or any other empirical risk minimization problem, and which can also be combined with our method.  To ensure a fair comparison, we closely followed the same settings as in \cite{bui_UDR_2022unified}; see Section \ref{app:setup}. 


{\bf Results:}  Here we compare the performance of the methods under the AutoAttack \cite{croce2020reliable} and PGD${}^{200}$    attacks as well as their performance when not under attack and report  the results in Table~\ref{tab:img_goal2_main}.  We used the implementation of UDR-PGD from \cite{bui_UDR_2022unified} to obtain the corresponding results in Table~\ref{tab:img_goal2_main}. Other values for the benchmark methods in Table~\ref{tab:img_goal2_main} are the same as the ones reported in \cite{bui_UDR_2022unified}. Several observations are made from these results: First, on both datasets our proposed method attains higher robust accuracy than all three baseline methods: adversarial training with PGD (PGD$-$AT) \cite{madry2018towardsrobustoptimization}, TRADES \cite{zhang2019theoreticallytradeOff}, and MART  \cite{wang2019improving_mart} under both AutoAttack and PGD${}^{200}$. More importantly, we find that ARMOR outperforms all UDR variations proposed in \cite{bui_UDR_2022unified}: UDR$-$PGD, UDR$-$TRADES, and UDR$-$MART for both AutoAttack and $PGD^{200}$. We note that stronger defense methods often have a side effect of reducing accuracy when not under attack; this behavior can be seen in the `No Attack' columns of Table \ref{tab:img_goal2_main}, and it is true for the competing methods as well as our own. In particular, on CIFAR-10 and CIFAR-100 all defense methods underperform standard non-robust training with cross-entropy loss (CE) when not under attack. However, when the goal is a robust model, the benefits of using a stronger defense method can greatly outweigh the reduction in no-attack accuracy; in particular, while CE achieves the best no-attack performance, it fails completely (0\% accuracy) against the strong attacks employed here. Therefore, for our purposes, we consider the results in the attack columns (AutoAttack and PGD${}^{200}$) to be of primary importance when evaluating the performance of the methods; in particular, we consider the performance under AutoAttack, which tests against an ensemble of powerful attacks, to be the best indicator of model robustness in our experiments. We find that augmenting with $ARMOR_D$ leads to a maximum improvement of 1.9\% and 2.1\%   against  AutoAttack on CIFAR-10 and CIFAR-100 respectively. 
These observations indicate that the ARMOR's re-weighting mechanism is an effective tool for enhancing the  robustness of an empirical risk minimization problem in adversarial environments. These results also illustrate the flexibility of the $ARMOR_D$ framework, as it is able to successfully augment all three of the existing adversarial training methods considered here and improve their performance in all cases. Finally, we emphasize that layering $ARMOR_D$ on top of these other methods requires only a marginal increase in the  computational effort, due to it adding at most two  real parameters to the optimizations \eqref{eq:OT_reg_Df_DRO}-\eqref{eq:OT_reg_KL_DRO}, as compared to the pure OT-based adversarial training methods \eqref{eq:classical_robust_optimization} and \eqref{eq:OT_DRO_Bui}.

\begin{table}[H]
\centering
\caption{{\bf Enhancing adversarial robustness on CIFAR-10 and CIFAR-100:} Here  $ARMOR$ uses natural samples alongside the adversarial samples, as in \eqref{eq:PGD_T_M_losses} and as described in Section \ref{app:adv+nat_description}. Best performance against attack is shown in bold font. } 
\label{tab:img_goal2_main}
\setlength\tabcolsep{2pt}
\centering
\begin{tabular}{|l|l|ll|l|ll|}
\hline
 & \multicolumn{3}{c|}{CIFAR-10} & \multicolumn{3}{c|}{CIFAR-100}\\ \hline
{\bf Defense} & No Attack  & \multicolumn{1}{l}{AutoAttack} & \multicolumn{1}{l|}{PGD${}^{200}$} & No Attack  & \multicolumn{1}{l}{AutoAttack} & \multicolumn{1}{l|}{PGD${}^{200}$} \\ \hline
CE & 95.2\% & \multicolumn{1}{l}{ 0.0\%} & \multicolumn{1}{l|}{0.0\%} &  77.50\% & 0.0\%& 0.0\%\\ \hline \hline
PGD$-$AT & {86.4}\% & \multicolumn{1}{l}{ 42.5\%} & \multicolumn{1}{l|}{46.0\%} &  72.4\% & 39.3\%& 41.7\%\\ \hline
UDR & 81.7\% & \multicolumn{1}{l}{ 48.4\%} & \multicolumn{1}{l|}{52.9\%} & 73.5\% & 41.9\%& 45.1\%\\ \hline
\textbf{ARMOR$-$UDR} & 80.3\% & \multicolumn{1}{l}{ \bf 48.6\%} & \multicolumn{1}{l|}{\bf 53.6\%} & 70.5\%& \textbf{43.8\%}& \textbf{46.7\%}\\
\hline\hline
$TRADES$ & 80.8\% & \multicolumn{1}{l}{ 49.1\%} & \multicolumn{1}{l|}{51.9\%} & 68.1\% & 46.7\%& 49.7\%\\ \hline
UDR$-$TRADES & {84.4}\% & \multicolumn{1}{l}{49.9\%} & \multicolumn{1}{l|}{53.6\%} & 69.6\% & 47.8\% &  49.9\%\\ \hline
\textbf{ARMOR$-$TRADES} & 80.8\%  &\multicolumn{1}{l}{\bf 51.8\%} & \multicolumn{1}{l|}{\bf 53.8\%} & 67.1\% & \textbf{49.9\%} & \textbf{50.3\%}\\
\hline\hline
MART & 81.9\%  &\multicolumn{1}{l}{ 48.2\%} & \multicolumn{1}{l|}{53.3\%} & 68.1\% & 44.8\%& 49.8\% \\
\hline
UDR$-$MART & 80.1\%  &\multicolumn{1}{l}{ 49.1\%} & \multicolumn{1}{l|}{54.1\%} & 67.5\%& 48.5\%& 52.0\% \\ \hline
\textbf{ARMOR$-$MART} & 81.0\%  &\multicolumn{1}{l}{\bf 50.6\%} & \multicolumn{1}{l|}{\bf{56.2}\%} & 65.5\%& \textbf{50.2\%} & \textbf{52.1\%} \\ \hline
\end{tabular}
\end{table}

\section{Properties of the OT-Regularized Divergences: Rigorous Statements and Proofs}\label{sec:proofs}
In this section we rigorously develop the definition and properties of the optimal-transport-regularized divergences that were introduced formally in Section \ref{sec:formal_derivation} above. Here we will let $\mathcal{Z}$ be a Polish space (i.e., a complete separable metric space) with its Borel $\sigma$-algebra (denoted $\mathcal{B}(\mathcal{Z})$) and  $\mathcal{P}(\mathcal{Z})$ will denote the space of Borel probability measures on $\mathcal{Z}$. A {\bf pre-divergence} will be a mapping $D:\mathcal{P}(\mathcal{Z})\times\mathcal{P}(\mathcal{Z})\to[0,\infty]$ such that $D(\mu\|\mu)=0$ for all $\mu\in\mathcal{P}(\mathcal{Z})$.  We will say that $D$ has the {\bf divergence property} if $D(\mu\|\nu)=0$ iff $\mu=\nu$. A {\bf cost function} on $\mathcal{Z}$ will be a lower semicontinuous (LSC) function $c:\mathcal{Z}\times\mathcal{Z}\to[0,\infty]$.  The associated {\bf optimal-transport (OT) cost} is defined by $C:\mathcal{P}(\mathcal{Z})\times\mathcal{P}(\mathcal{Z})\to[0,\infty]$, \begin{align}
    C(\mu,\nu)\coloneqq \inf_{\substack{\pi\in\mathcal{P}(\mathcal{Z}\times\mathcal{Z}):\\ \pi_1=\mu,\pi_2=\nu}}\int cd\pi\,,
    \end{align}
    where  $\pi_1,\pi_2$ denote the marginal distributions; see, e.g., \cite{villani2008optimal} for background on optimal transport.  It is a simple exercise to check that  if $c(z,z)=0$  for all $z$ then  $C(\mu,\mu)=0$ for all $\mu$ and if $c(z,\tilde z)=0$ iff $z=\tilde z$ then $C(\mu,\nu)=0$ iff $\mu=\nu$. Also recall that $C$ is convex and is LSC in the product of  Prokhorov metric topologies. This follows from Kantorovich duality; see Theorem 5.10 in \cite{villani2008optimal}. All subsequent topological statements regarding probability distributions  will refer to the Prokhorov metric topology (i.e., the topology of weak convergence).

Given the above ingredients we now define the class of optimal-transport regularized divergences that are employed in this work.
\begin{definition}
Let $D$ be a pre-divergence and $c$ a cost function. The {\bf OT-regularized divergence}, $D^c:\mathcal{P}(\mathcal{Z})\times\mathcal{P}(\mathcal{Z})\to[0,\infty]$, is defined by 
\begin{align}\label{eq:D_c_def}
D^c(\nu\|\mu)\coloneqq \inf_{\eta\in\mathcal{P}(\mathcal{Z})}\{D(\eta\|\mu)+C(\eta,\nu)\}\,.
\end{align}
\end{definition}
Above we have frequently  referred to $D(\eta\|\mu)$ as an information divergence, meaning it is computable in terms of $d\eta/d\mu$, however our rigorous development in this section will be stated more generally. { In particular, in Section \ref{sec:CC_general_D} we provide a theorem that allows for the explicit construction of a large collection of new divergences to which our theory can be applied. }

The experiments in Section \ref{sec:experiments} focused { primarily} on the $f$-divergences, $D=D_f$, an important family of information divergences. Throughout Section \ref{app:properties}  we provide remarks indicating how the theorems proven here can be specifically applied to OT-regularized $f$-divergences.  We use the following definition of $f$-divergences.
\begin{definition}\label{def:f_div}
For $a,b\in[-\infty,\infty]$ that satisfy $-\infty\leq a<1<b\leq\infty$ we define $\mathcal{F}_1(a,b)$ to be the set of convex functions $f:(a,b)\to\mathbb{R}$ with $f(1)=0$.  For $f\in\mathcal{F}_1(a,b)$, the corresponding {\bf $f$-divergence} between $\nu,\mu\in\mathcal{P}(\mathcal{Z})$ is defined by
\begin{align}\label{eq:f_div_def}
D_f(\nu\|\mu)=\begin{cases} 
     E_P[f(d\nu/d\mu)], & \nu\ll \mu\\
      \infty, &\nu\not\ll \mu
   \end{cases}\,,
\end{align}
where the definition of $f$ in \eqref{eq:f_div_def} is extended to $[a,b]$ by continuity and is set to $\infty$ on $[a,b]^c$.
\end{definition}
\begin{remark}
For certain choices of $f$ one can assign a meaningful finite value to $D_f(\nu\|\mu)$  when $\nu\not\ll\mu$ \cite{LieseVajda} but the definition \eqref{eq:f_div_def} is more convenient for our purposes.  That alternative definition agrees with \eqref{eq:f_div_def} for the choices of $f$ used in the experiments in Section \ref{sec:experiments}.
\end{remark}
In our numerical experiments we use the KL divergence, defined using $f_{KL}(t)=t\log(t)$, and the $\alpha$-divergences, defined using 
\begin{align}\label{eq:f_alpha_def}
f_\alpha(t)=      \frac{t^\alpha-1}{\alpha(\alpha-1)},\,\,\,\,\, \alpha>1\,.\end{align}
The Legendre transform of $f_\alpha$ will also be required:
\begin{align}\label{eq:f_alpha_star}
f_\alpha^*(t)=\alpha^{-1}(\alpha-1)^{\alpha/(\alpha-1)}\max\{t,0\}^{\alpha/(\alpha-1)}+\frac{1}{\alpha(\alpha-1)},\,\,\,\,\,\alpha>1\,.
\end{align}

\subsection{Properties of the OT-Regularized Divergences}\label{app:properties}

Here we prove a number of key properties of the OT-regularized divergences.
\begin{lemma}[Convexity]\label{lemma:D_c_convex_property}
Let $D$ be a pre-divergence, $c$  be a cost function, and $\mu\in\mathcal{P}(\mathcal{Z})$.   If $P\mapsto D(P\|\mu)$ is convex then  $P\mapsto D^c(P\|\mu)$ is convex.
\end{lemma}
\begin{remark}\label{remark:Df_convex}
$f$-divergences satisfy the necessary convexity property.  In fact, the map $(Q,P)\mapsto D_f(Q\|P)$ is convex for all $f\in\mathcal{F}_1(a,b)$. This follows from the variational representation of $f$-divergences; see   \cite{Nguyen_Full_2010}, \cite{Broniatowski} and also Proposition B.1 in \cite{JMLR:v23:21-0100}.  
\end{remark}
\begin{proof}
$C$ is convex on $\mathcal{P}(\mathcal{Z})\times \mathcal{P}(\mathcal{Z})$ and so $(\eta,\nu)\mapsto D(\eta\|\mu)+C(\eta,\nu)$ is convex. Therefore the infimum over $\eta$ is convex in $\nu$.
\end{proof}

\begin{lemma}[Pre-Divergence Property]\label{lemma:D_c_pre_div_property}
Let $D$ be a pre-divergence and $c$  be a cost function that satisfies $c(z,z)=0$ for all $z\in\mathcal{Z}$.  Then $D^c$ is a pre-divergence.
\end{lemma}
\begin{proof} The definition \eqref{eq:D_c_def} of $D^c$ is clearly non-negative.  
We need to show that $D^c(\mu\|\mu)=0$ for all $\mu\in\mathcal{P}(\mathcal{Z})$.  To do this we bound the infimum in \eqref{eq:D_c_def} by the value at $\eta=\mu$ to obtain
\begin{align}
0\leq D^c(\mu\|\mu)\leq D(\mu\|\mu)+C(\mu,\mu)=0\,,
\end{align}
where $C(\mu,\mu)=0$ follows from the assumption on $c$.
\end{proof}

\begin{theorem}[Divergence Property]\label{thm:D_c_div_property}
Let $D$ be a pre-divergence and $c$  be a cost function that satisfy the following properties.
\begin{enumerate}
\item If $D(\mu_n\|\mu)\to 0$ then $\mu_n\to \mu$ weakly.
\item $c(z,\tilde z)=0$ iff $z=\tilde z$.
\end{enumerate}
Then $D^c$ has the divergence property.
\end{theorem}
\begin{remark}\label{remark:D_c_div_property}
In Theorem \ref{thm:f_div_setwise} of Appendix \ref{app:additional_proofs} we show that the $f$-divergences satisfy the weak convergence property under mild assumptions and hence this theorem can be applied to OT-regularized $f$-divergences.
\end{remark}
\begin{proof}
Lemma \ref{lemma:D_c_pre_div_property} implies $D^c$ is a pre-divergence, therefore we only need to show that  $D^c(\nu\|\mu)=0$ implies $\nu=\mu$.  By the definition \eqref{eq:D_c_def}, if $D^c(\nu\|\mu)=0$ then there exists a sequence $\eta_n\in\mathcal{P}(\mathcal{Z})$ such that $D(\eta_n\|\mu)+C(\eta_n,\nu)\to 0$, i.e., $D(\eta_n\|\mu)\to 0$ and $C(\eta_n,\nu)\to 0$.  By the weak convergence property of $D$ we have $\eta_n\to \mu$ weakly. $C$ is LSC, therefore
\begin{align}
C(\mu,\nu)\leq\lim_{n\to\infty}C(\eta_n,\nu)=0
\end{align}
and we can conclude that $C(\mu,\nu)=0$. Noting the assumption on $c$, we can then conclude $\mu=\nu$. 
\end{proof}

Next we provide conditions under which the infimum in \eqref{eq:D_c_def} has a (unique) solution. 
\begin{theorem}\label{thm:inf_conv_solution}
Let $D$ be a pre-divergence, $c$ a cost function, and $\mu,\nu\in \mathcal{P}(\mathcal{Z})$. If  the mapping $P\mapsto D(P\|\mu)$ is LSC and has compact sublevel sets (i.e., $\{P:D(P\|\mu)\leq M\}$ is compact for all $M\in\mathbb{R}$) then there exists $\eta_*\in\mathcal{P}(\mathcal{Z})$ such that
\begin{align}\label{eq:existence_intermediate}
D^c(\nu\|\mu)=D(\eta_*\|\mu)+C(\eta_*,\nu)\,.
\end{align}
If $P\mapsto D(P\|\mu)$ is strictly convex on the set where it is finite and  $D^c(\nu\|\mu)<\infty$ then this $\eta_*$ is unique.
\end{theorem}
\begin{remark}\label{remark:inf_conv_solution}
For $f\in\mathcal{F}_1(a,b)$ the $f$-divergences $D_f(\cdot\|\mu)$ are LSC and have compact sublevel sets for all $\mu$, provided that $f^*$ is finite everywhere; see Corollary B.2 and Lemma B.5 in \cite{JMLR:v23:21-0100}. If $f$ is strictly convex on $(a,b)$ then $D_f(\cdot\|\mu)$ is strictly convex on the set where it is finite; see Lemma B.6 in \cite{JMLR:v23:21-0100}. Therefore Theorem \ref{thm:inf_conv_solution} can be applied to OT-regularized $f$-divergences for appropriate choices of $f$.
\end{remark}
\begin{proof}
If $D^c(\nu\|\mu)=\infty$ then the definition \eqref{eq:D_c_def} implies that \eqref{eq:existence_intermediate} holds for all $\eta_*\in\mathcal{P}(\mathcal{Z})$.  Now consider the case where $D^c(\nu\|\mu)<\infty$. Take $\eta_n$ such that $D^c(\nu\|\mu)=\lim_n(D(\eta_n\|\mu)+C(\eta_n,\nu))$. Without loss of generality we can assume that $D(\eta_n\|\mu)\leq D^c(\nu\|\mu)+1<\infty$ for all $n$, i.e., $\eta_n$ are all contained in a sublevel set of $D(\cdot\|\mu)$, which is compact by assumption. Therefore there exists a weakly convergent subsequence $\eta_{n_j}\to \eta_*$. Lower semicontinuity of $D(\cdot\|\mu)$ and of $C$ then implies $\liminf_j D(\eta_{n_j}\|\mu)\geq D(\eta_*\|\mu)$ and $\liminf_j C(\eta_{n_j},\nu)\geq C(\eta_*,\nu)$. Therefore
\begin{align}
D^c(\nu\|\mu)=\lim_j(D(\eta_{n_j}\|\mu)+C(\eta_{n_j},\nu))\geq D(\eta_*\|\mu)+C(\eta_*,\nu)\,.
\end{align}
The reverse inequality is obvious from the definition of $D^c$, hence we can conclude
\begin{align}
D^c(\nu\|\mu)=D(\eta_*\|\mu)+C(\eta_*,\nu)\,.
\end{align}
Now consider the case where  $P\mapsto D(P\|\mu)$ is also strictly convex on the set where it is finite.  Suppose there exist distinct $\eta_{*,1},\eta_{*,2}\in\mathcal{P}(\mathcal{Z})$ such that 
\begin{align}
D^c(\nu\|\mu)=D(\eta_{*,1}\|\mu)+C(\eta_{*,1},\nu)=D(\eta_{*,2}\|\mu)+C(\eta_{*,2},\nu)\,.
\end{align}
Letting $\eta_*=\frac{1}{2}(\eta_{*,1}+\eta_{*,2})$ we can use convexity of $C$ and strict convexity of $D(\cdot\|\mu)$ to compute
\begin{align}
D^c(\nu\|\mu)\leq& D(\eta_*\|\mu)+C(\eta_*,\nu)\\
<& \frac{1}{2}D(\eta_{*,1}\|\mu)+\frac{1}{2}D(\eta_{*,2}\|\mu)+\frac{1}{2}C(\eta_{*,1},\nu)+\frac{1}{2}C(\eta_{*,2},\nu)\notag\\
=&\frac{1}{2}D^c(\nu\|\mu)+\frac{1}{2}D^c(\nu\|\mu)=D^c(\nu\|\mu)\,.\notag
\end{align}
This is a contradiction, therefore we can conclude the optimizer is unique.
\end{proof}

Using Theorem \ref{thm:inf_conv_solution} we can prove $D^c(\cdot\|\mu)$ is LSC; see Remark \ref{remark:inf_conv_solution} for the application to OT-regularized $f$-divergences.
\begin{theorem}[Lower Semicontinuity]\label{thm:Dc_LSC}
Let $D$ be a pre-divergence, $c$ a cost function, $\mu\in\mathcal{P}(\mathcal{Z})$, and assume that $D(\cdot\|\mu)$ is LSC and has compact sublevel sets.  Then $\nu\mapsto D^c(\nu\|\mu)$ is LSC.
\end{theorem}
\begin{proof}
Let $\nu_n,\nu\in\mathcal{P}(\mathcal{Z})$ with $\nu_n\to \nu$ weakly  and define $M\coloneqq\liminf_n D^c(\nu_n\|\mu)$. If $M=\infty$ then we clearly have $\liminf_n D^c(\nu_n\|\mu)\geq D^c(\nu\|\mu)$ so suppose $M<\infty$. Therefore, fixing $\delta>0$, there exists $N$ such that for all $n\geq N$ we have $\inf_{j\geq n}D^c(\nu_j\|\mu)<M+\delta$.  Hence we can construct a subsequence $j_k$ such that $D^c(\nu_{j_k}\|\mu)<M+\delta$ for all $k$. Theorem \ref{thm:inf_conv_solution} implies that there exists $\eta_{*,n}$ such that $D^c(\nu_n\|\mu)=D(\eta_{*,n}\|\mu)+C(\eta_{*,n},\nu_n)$ for all $n$, and so
\begin{align}
M+\delta>D^c(\nu_{j_k}\|\mu)=D(\eta_{*,j_k}\|\mu)+C(\eta_{*,j_k},\nu_{j_k})
\end{align}
for all $k$.  In particular, the $\eta_{*,j_k}$ are contained in the  compact sublevel set $\{D(\cdot\|\mu)\leq M+\delta\}$.  Therefore there exists a convergent subsequence $\eta_{*,j_{k_\ell}}\to \eta_*$.  Lower semicontinuity of $D(\cdot\|\mu)$ and $C$ then implies 
\begin{align}
M+\delta\geq& \liminf_\ell(D(\eta_{*,j_{k_\ell}}\|\mu)+C(\eta_{*,j_{k_\ell}},\nu_{j_{k_\ell}}))\geq \liminf_\ell D(\eta_{*,j_{k_\ell}}\|\mu)+\liminf_\ell C(\eta_{*,j_{k_\ell}},\nu_{j_{k_\ell}}))\notag\\
\geq& D(\eta_{*}\|\mu)+C(\eta_{*},\nu)\geq D^c(\nu\|\mu)\,.
\end{align}
Taking $\delta\to 0^+$ and recalling the definition of $M$ completes the proof.
\end{proof}

Finally, we prove a pair of results showing that $D^c$ reduces to either $D$ or $C$ in certain limits. Therefore one can think of $D^c$ as a type of interpolation between $D$ and $C$. To apply these theorems to the case $D=D_f$, see Remarks \ref{remark:D_c_div_property} and \ref{remark:inf_conv_solution}.
\begin{theorem}[Interpolation]\label{thm:limit_D_c_to_C}
Let $D$ be a pre-divergence, $c$ be a cost function, and $\mu,\nu\in\mathcal{P}(\mathcal{Z})$ that satisfy the following.
\begin{enumerate}
\item The mapping $P\mapsto D(P\|\mu)$ is LSC and has compact sublevel sets.
\item  $D(\mu_n\|\mu)\to 0$ implies  $\mu_n\to \mu$ weakly.
\end{enumerate}
  For $r>0$ define the cost function $c_r=rc$.  Then
\begin{align}
\lim_{r\to0^+}r^{-1}D^{c_r}(\nu\|\mu)=C(\mu,\nu) \,\,\text{ for all } \mu,\nu\in \mathcal{P}(\mathcal{Z}).
\end{align}
\end{theorem}
\begin{proof}
From the definitions we have
\begin{align}
r^{-1}D^{c_r}(\nu\|\mu)=\inf_{\eta\in\mathcal{P}(\mathcal{Z})}\{r^{-1}D(\eta\|\mu)+C(\eta,\nu)\}
\end{align}
and the right-hand side is non-increasing in $r$.  Therefore for $r_n\searrow 0$ we have
\begin{align}
\lim_n r_{n}^{-1}D^{c_{r_n}}(\nu\|\mu)=\sup_nr_{n}^{-1}D^{c_{r_n}}(\nu\|\mu) \leq C(\mu,\nu)\,,
\end{align}
where the inequality comes from bounding \eqref{eq:D_c_def} by its value at $\eta=\mu$.  We will show that the assumption that this inequality is strict leads to a contradiction, which will complete the proof. If the inequality is strict then $D^{c_{r_n}}(\nu\|\mu)<\infty$ for all $n$ and Theorem \ref{thm:inf_conv_solution} implies the existence of $\eta_{*,n}$ such that
\begin{align}
D(\eta_{*,n}\|\mu)\leq D(\eta_{*,n}\|\mu)+r_nC(\eta_{*,n},\nu)=D^{c_{r_n}}(\nu\|\mu)\leq r_n \sup_mr_{m}^{-1}D^{c_{r_m}}(\nu\|\mu)<\infty\,.
\end{align}
Taking $n\to\infty$ we see that $D(\eta_{*,n}\|\mu)\to 0$ and therefore $\eta_{*,n}\to \mu$ weakly.  $C$ is LSC, therefore $\liminf_n C(\eta_{*,n},\nu)\geq C(\mu,\nu)$. Combining these we have
\begin{align}
C(\mu,\nu)>&\sup_n r_n^{-1}D^{c_{r_n}}(\nu\|\mu)\geq \liminf_n (r_n^{-1}D(\eta_{*,n}\|\mu)+C(\eta_{*,n},\nu))\\
\geq &\liminf_nC(\eta_{*,n},\nu)\geq C(\mu,\nu)\,.\notag
\end{align}
This is a contradiction, hence the proof is complete.
\end{proof}

\begin{theorem}[Interpolation]\label{thm:limit_D_c_to_D}
Let $D$ be a pre-divergence, $c$ be a cost function, and $\mu,\nu\in\mathcal{P}(\mathcal{Z})$ that satisfy the following.
\begin{enumerate}
\item The mapping $P\mapsto D(P\|\mu)$ is LSC and has compact sublevel sets.
\item  $c(z,\tilde z)=0$ iff $z=\tilde z$.
\end{enumerate}
  For $r>0$ define the cost function $c_r=rc$.   Then
\begin{align}
\lim_{r\to \infty}D^{c_r}(\nu\|\mu)=D(\nu\|\mu)\,\, \text{ for all } \mu,\nu\in \mathcal{P}(\mathcal{Z}).
\end{align}
\end{theorem}
The proof of this second interpolation result uses similar ideas to that of Theorem \ref{thm:limit_D_c_to_C}, hence we present its proof in Appendix \ref{supp:interp_thm}.

\subsection{DRO Using OT-Regularized Divergences}\label{app:DRO_identity_proofs}
In this section we provide rigorous proofs of the key identities that transform the DRO problem over OT-regularized-divergence neighborhoods into a computationally tractable form. This will involve the construction of regularized loss functions, as defined below.
\begin{definition}
Given a loss function $\mathcal{L}:\mathcal{Z}\to[-\infty,\infty]$ we define the corresponding family of {\bf OT-regularized losses} by
\begin{align}\label{eq:regularized_loss_def}
   \mathcal{L}^c_{\lambda}(z)\coloneqq  \sup_{\tilde z\in\mathcal{Z}}\{ \mathcal{L} (\tilde z) -  \lambda c(z,\tilde z)\}\,,\,\,\,\lambda>0\,,
\end{align}
where we employ the convention $\infty-\infty\coloneqq-\infty$. $\mathcal{L}^c_{\lambda}$  is known as the $c$-transform in the optimal transport literature; see Definition  5.2 in \cite{villani2008optimal}. 
\end{definition}
\begin{remark}
From a mathematical perspective, the convention $\infty-\infty\coloneqq-\infty$ is motivated by the proof of Theorem \ref{thm:Dc_convex_conjugate_general_P} below.  It also  coincides with the behavior one intuitively wants  based on  viewing the maximization in \eqref{eq:regularized_loss_def} as the construction of a new sample $\tilde z$ that is adversarial to the original sample $z$.  If the transport cost $c(z,\tilde z)=\infty$ then one should view $\tilde z$ as impossible to reach when starting from $z$ and so $\tilde z$ should not be a valid adversarial sample to pair with $z$, even if $\mathcal{L}(\tilde z)=\infty$.  Therefore such $\tilde z$'s should be excluded from the maximization in \eqref{eq:regularized_loss_def}, which corresponds to defining $\infty-\infty\coloneqq-\infty$.
\end{remark}

In Section \ref{sec:experiments}, we used DRO as a tool for enhancing adversarial robustness, and thus we considered distribution neighborhoods of the form $\{Q:D^c(Q\|P_n)\leq \epsilon\}$, where the baseline distribution is an empirical distribution $P_n$.  However, for other purposes it is useful to have a proof of the DRO identity for the neighborhoods $\{Q:D^c(Q\|P)\leq \epsilon\}$ with a general baseline distribution $P$; so we study this more general problem next. A key tool will be the following interchangeability result, which has previously been used in Wasserstein and OT DRO; see the discussion in \cite{2022arXiv220500362Z}. For completeness, we provide a proof of the version used in this work in  Appendix \ref{supp:interchangability}. Our proof mimics the strategy used for the more general result stated in \cite{2022arXiv220500362Z}.  In the following, $\overline{\mu}$  denotes the completion  of the measure $\mu$ and $\mathcal{B}(\mathcal{Z})_*$ denotes the $\sigma$-algebra of universally measurable sets (relative to $\mathcal{B}(\mathcal{Z})$).
\begin{lemma}[Interchangeability]\label{lemma:interchangability}
Let $\mu\in\mathcal{P}(\mathcal{Z})$ and  $\phi:\mathcal{Z}\times \mathcal{Z}\to[-\infty,\infty]$ be  measurable.   Then $z\mapsto\sup_{\tilde z\in \mathcal{Z}}\phi(z,\tilde z)$ is a $\mathcal{B}(\mathcal{Z})_*$-measurable function  and
\begin{align}\label{eq:interchangability}
\sup_{\pi\in\mathcal{P}(\mathcal{Z}\times \mathcal{Z}):\pi_1=\mu} E_\pi[\phi]=\int \sup_{\tilde z\in \mathcal{Z}}\phi(z,\tilde z)\overline{\mu}(dz)\,.
\end{align}
\end{lemma}

Now we derive a formula that relates the convex conjugate of $D^c(\cdot\|P)$ to the convex conjugate of $D(\cdot\|P)$.  This is a useful result in its own right and is a key ingredient in solving the DRO problem.
\begin{theorem}\label{thm:Dc_convex_conjugate_general_P}
Suppose we have the following:
\begin{enumerate}
\item A measurable function $\mathcal{L}:\mathcal{Z}\to[-\infty,\infty]$ that is  bounded below or is bounded above.
\item A distribution $P\in\mathcal{P}(\mathcal{Z})$.
\item A pre-divergence, $D$, such that $D(\cdot\|P)$ is convex.
\item  A  cost function, $c$, that satisfies $c(z,z)=0$ for all $z\in\mathcal{Z}$.
\end{enumerate}
Then for $\lambda>0$ we have
\begin{align}\label{eq:Dc_convex_conjugate}
\sup_{\substack{Q\in\mathcal{P}(\mathcal{Z}):\\ D^c(Q\|P)<\infty}}\{E_Q[\mathcal{L}]-\lambda D^c(Q\|P)\}=\sup_{\substack{Q\in\mathcal{P}(\mathcal{Z}):\\D(Q\|P)<\infty}}\{  E_{{Q}}[ \mathcal{L}_\lambda^c]-  \lambda D(Q\|P)\}\,,
\end{align}
where $\mathcal{L}^c_{\lambda}$ (defined in Eq.~\ref{eq:regularized_loss_def}) is a  universally measurable function. 
\end{theorem}
\begin{remark}
     In Theorem \ref{thm:Dc_convex_conjugate_general_P} and in the following, when it is convenient for simplifying notation we use the same symbol to denote a probability measure and its completion, as the correct interpretation is easily discovered by examining the measurably of the integrand.
\end{remark}
\begin{proof}
Universal measurability of  $\mathcal{L}^c_{\lambda}$ follows from the interchangeability result, Lemma \ref{lemma:interchangability}. To prove \eqref{eq:Dc_convex_conjugate}, first suppose that $\mathcal{L}$ is bounded above.  Using the definitions of $D^c$ and $C$ we can compute
\begin{align}\label{eq:Dc_CC_bounded_above}
&\sup_{Q\in\mathcal{P}(\mathcal{Z}):D^c(Q\|P)<\infty}\left\{E_Q[\mathcal{L}]-\lambda D^c(Q\|P)\right\}\\
=&\sup_{Q\in\mathcal{P}(\mathcal{Z})}\left\{E_Q[\mathcal{L}]-\lambda \inf_{\eta\in\mathcal{P}(\mathcal{Z})}\left\{D(\eta\|P)+C(\eta,Q)\right\}\right\}\notag\\
=&\lambda\sup_{\eta\in\mathcal{P}(\mathcal{Z})}\left\{\sup_{Q\in\mathcal{P}(\mathcal{Z})}\left\{ E_Q[\mathcal{L}/\lambda]-C(\eta,Q)\right\}-D(\eta\|P)\right\}\notag\\
=&\lambda\sup_{\eta\in\mathcal{P}(\mathcal{Z}):D(\eta\|P)<\infty}\left\{\sup_{Q\in\mathcal{P}(\mathcal{Z})}\left\{ \sup_{\pi:\pi_1=\eta,\pi_2=Q}\left\{ \int \lambda^{-1}\mathcal{L}(\tilde z) - c(z,\tilde z)\pi(dzd\tilde z)\right\}\right\}-D(\eta\|P)\right\}\notag\\
=&\lambda\sup_{\eta\in\mathcal{P}(\mathcal{Z}):D(\eta\|P)<\infty}\left\{ \sup_{\pi:\pi_1=\eta}\left\{ \int \lambda^{-1}\mathcal{L}(\tilde z) - c(z,\tilde z)\pi(dzd\tilde z)\right\}-D(\eta\|P)\right\}\notag\,.
\end{align}
Now use   the interchangability result, Lemma \ref{lemma:interchangability}, to obtain  
\begin{align}
&\sup_{Q\in\mathcal{P}(\mathcal{Z}):D^c(Q\|P)<\infty}\left\{E_Q[\mathcal{L}]-\lambda D^c(Q\|P)\right\}\\
=&\sup_{\eta\in\mathcal{P}(\mathcal{Z}):D(\eta\|P)<\infty}\left\{  \int \sup_{\tilde z\in\mathcal{Z}}\{\mathcal{L}(\tilde z) - \lambda c(z,\tilde z)\}{\eta}(dz)-\lambda D(\eta\|P)\right\}\,.\notag
\end{align}
The assumption that $\mathcal{L}$ is bounded above, and hence $E_Q[\mathcal{L}]\in[-\infty,\infty)$ for all $Q$,  ensured that $\infty-\infty$ was not encountered in the computations  \eqref{eq:Dc_CC_bounded_above}. Recalling the definition \eqref{eq:regularized_loss_def} this completes the proof when $\mathcal{L}$ is bounded above.

Now suppose $\mathcal{L}$ is bounded below. Define $\mathcal{L}_n(z)\coloneqq\min\{\mathcal{L}(z),n\}$, $n\in\mathbb{Z}^+$. These are bounded below uniformly in $n$ and so $\sup_n E_Q[\mathcal{L}_n]=E_Q[\mathcal{L}]$ for all $Q$ by the monotone convergence theorem. The $\mathcal{L}_n$ are  bounded above, hence we can use \eqref{eq:Dc_CC_bounded_above} to obtain
\begin{align}
&\sup_{Q\in\mathcal{P}(\mathcal{Z}):D^c(Q\|P)<\infty}\left\{E_Q[\mathcal{L}]-\lambda D^c(Q\|P)\right\}\\
=&\sup_n\sup_{Q\in\mathcal{P}(\mathcal{Z}):D^c(Q\|P)<\infty}\left\{E_Q[\mathcal{L}_n]-\lambda D^c(Q\|P)\right\}\notag\\
=&\sup_{Q\in\mathcal{P}(\mathcal{Z}):D(Q\|P)<\infty}\left\{\sup_n \int  (\mathcal{L}_n)^c_\lambda(z)Q(dz)-\lambda D(Q\|P)\right\}\notag\\
=&\sup_{Q\in\mathcal{P}(\mathcal{Z}):D(Q\|P)<\infty}\left\{ E_Q[\mathcal{L}^c_\lambda]-\lambda D(Q\|P)\right\}\,,\notag
\end{align}
where, noting that the functions $(\mathcal{L}_n)^c_\lambda$ are bounded below uniformly in $n$, we again used the monotone convergence theorem in the final equality.  We emphasize that the convention $\infty-\infty=-\infty$ is needed to justify the computation $\sup_n (\mathcal{L}_n)^c_\lambda(z)=\sup_{\tilde z}\{\sup_n \mathcal{L}_n(\tilde z)-\lambda c(z,\tilde z)\}=\mathcal{L}^c_\lambda(z)$ for all $z$. This proves the claim when $\mathcal{L}$ is bounded below and so the proof is complete.
\end{proof}

In particular, when $D=D_f$ is a $f$-divergence we can further evaluate the convex conjugate of $D_f$ to obtain a formula  that only involves expectations with respect to $P$.
\begin{corollary}\label{cor:CC_formula_OT_f_div}
Suppose we have the following:
\begin{enumerate}
\item A measurable function $\mathcal{L}:\mathcal{Z}\to[-\infty,\infty]$ that is  bounded below or is bounded above.
\item A distribution $P\in\mathcal{P}(\mathcal{Z})$ with $\mathcal{L}^-\in L^1(P)$, where $\mathcal{L}^-$ denotes the negative part of $\mathcal{L}$.
\item $f\in\mathcal{F}_1(a,b)$ with $a\geq 0$.
\item  A  cost function, $c$, that satisfies $c(z,z)=0$ for all $z\in\mathcal{Z}$.
\end{enumerate}
Then for $\lambda>0$ we have
\begin{align}
\sup_{Q\in\mathcal{P}(\mathcal{Z}):D_f^c(Q\|P)<\infty}\{E_Q[\mathcal{L}]-\lambda D_f^c(Q\|P)\}=\lambda\inf_{\rho\in\mathbb{R}}\{\rho+E_{{P}}[f^*(\lambda^{-1} \mathcal{L}^c_\lambda-\rho)]\}\,,
\end{align}
where the definition of $f^*$ is extended by $f^*(\pm\infty)\coloneqq\infty$.
\end{corollary}
The proof follows from combining  Theorem \ref{thm:Dc_convex_conjugate_general_P} with \eqref{eq:Df_Gibbs}; see Corollary \ref{cor:CC_formula_OT_f_div_supp} in Appendix \ref{app:additional_proofs} for details.

\subsection{A Tractable Reformulation of the DRO Problem}
In this section we use the above results to derive a tractable reformulation of the OT-regularized-divergence DRO problem. First we need one final lemma, regarding the finiteness of expectations as the distribution ranges over an OT-regularized-divergence neighborhood.
\begin{lemma}\label{lemma:finite_integrals}
Suppose we have the following:
\begin{enumerate}
\item A measurable function $\mathcal{L}:\mathcal{Z}\to [-\infty,\infty]$.
\item  A distribution $P\in\mathcal{P}(\mathcal{Z})$.
\item  A pre-divergence, $D$, such that $D(\cdot\|P)$ is convex.
\item A cost function, $c$, that satisfies $c(z,z)=0$ for all $z\in\mathcal{Z}$.
\end{enumerate}
Suppose there exists $\epsilon>0$ such that $\mathcal{L}\in L^1(Q)$ for all $Q\in\mathcal{P}(\mathcal{Z})$ that satisfy $D^c(Q\|P)\leq \epsilon$.  Then $\mathcal{L}\in L^1(Q)$ for all $Q$ that satisfy $D^c(Q\|P)<\infty$.
\end{lemma}
\begin{proof}
The assumptions imply that $D^c$ is a pre-divergence (see Lemma \ref{lemma:D_c_pre_div_property}) and $D^c(\cdot\|P)$ is convex (see Lemma \ref{lemma:D_c_convex_property}). Take any $Q$ with $D^c(Q\|P)<\infty$ and define $Q_t=tQ+(1-t)P$ for $t\in(0,1)$.  By convexity and the pre-divergence property we have $D^c(Q_t\|P)\leq tD^c( Q\|P)$.  We assumed $D^c( Q\|P)<\infty$, hence there exists $t\in(0,1)$ with $D^c(Q_t\|P)\leq \epsilon$. This implies $\mathcal{L}\in L^1(Q_t)$ and so $\infty >E_{Q_t}[|\mathcal{L}|]=tE_{ Q}[|\mathcal{L}|]+(1-t)E_P[|\mathcal{L}|]$. We have $t\in(0,1)$, therefore we can conclude $E_{ Q}[|\mathcal{L}|]<\infty$ as claimed.
\end{proof}

We are now ready to consider the DRO problem for general $P$. We also allow for an explicit $D^c$ penalty term, in addition to maximizing over the distribution neighborhood; { such a penalty term can be useful for alleviating numerical issues that can arise when $\lambda\to 0$.}
\begin{theorem}\label{thm:DRO_gen_P_gen_D}
Suppose we have the following:
\begin{enumerate}
\item A measurable function $\mathcal{L}:\mathcal{Z}\to[-\infty,\infty]$ that is  bounded below or is bounded above.
\item A distribution $P\in\mathcal{P}(\mathcal{Z})$ with $\mathcal{L}^-\in L^1(P)$.
\item A pre-divergence, $D$, such that $D(\cdot\|P)$ is convex.
\item  A  cost function, $c$, that satisfies $c(z,z)=0$ for all $z\in\mathcal{Z}$.
\end{enumerate}
Then for $\epsilon>0$, $\kappa\geq 0$ we have
\begin{align}\label{eq:DRP_general_P_unbounded_loss}
&\sup_{Q:D^c(Q\|P)\leq \epsilon}\{E_Q[\mathcal{L}]-\kappa D^c(Q\|P)\}\\
=&\inf_{\lambda>0}\{\lambda \epsilon+(\lambda+\kappa)\sup_{Q\in\mathcal{P}(\mathcal{Z}):D(Q\|P)<\infty}\{E_{{Q}}[(\lambda+\kappa)^{-1} \mathcal{L}_{\lambda+\kappa}^c]- D(Q\|P)\}\}\,.\notag
\end{align} 
\end{theorem}
\begin{proof}
We will show that
\begin{align}\label{eq:strong_duality_general_P}
\sup_{Q:D^c(Q\|P)\leq \epsilon}\{E_Q[\mathcal{L}]-\kappa D^c(Q\|P)\}=\inf_{\lambda>0}\{\lambda\epsilon+\sup_{Q:D^c(Q\|P)<\infty}\{E_Q[\mathcal{L}]-(\lambda+\kappa)D^c(Q\|P)\}\}\,.
\end{align}
Combining this with the result of Theorem \ref{thm:Dc_convex_conjugate_general_P} will then complete the proof. If there exists $Q$ such that $D^c(Q\|P)\leq \epsilon$ and $E_Q[\mathcal{L}^+]=\infty$ then it is straightforward to see that both sides of \eqref{eq:strong_duality_general_P} equal $\infty$.  Therefore it suffices to consider the case where $E_Q[\mathcal{L}^+]<\infty$ for all $Q$ satisfying $D^c(Q\|P)\leq \epsilon$.  Applying  Lemma \ref{lemma:finite_integrals} to $\mathcal{L}^+$ then implies $\mathcal{L}^+\in L^1(Q)$ for all $Q$ satisfying $D^c(Q\|P)<\infty$. Therefore the $F:Q\mapsto E_Q[\mathcal{L}]-\kappa D^c(Q\|P)$ is a concave map from $\{Q:D^c(Q\|P)<\infty\}$ to $[-\infty,\infty)$ and $Q\mapsto D^c(Q\|P)$ is a convex constraint. $P$ satisfies $F[P]\in\mathbb{R}$ and $D^c(P\|P)<\epsilon$. Therefore Slater's constraint qualification condition holds (see, e.g., Theorem 3.11.2  in \cite{ponstein2004approaches}). This implies strong duality, i.e., that the equality \eqref{eq:strong_duality_general_P} holds. We note that the infimum can be restricted to $\lambda>0$ (rather than $\lambda\geq 0$) due to the lower bound on the constraint function,   $D^c(\cdot\|P)\geq 0$. This proves the claim.
\end{proof}

If $D$ is a $f$-divergence, the convex conjugate term (i.e., the supremum over $Q$) in \eqref{eq:DRP_general_P_unbounded_loss} can be evaluated   as in Corollary \ref{cor:CC_formula_OT_f_div}, resulting in a two-dimensional convex optimization problem. The details can be found in Appendix \ref{supp:reformulation_OT_f_DRO}.
\begin{corollary}\label{cor:OT_f_div_DRO}
Suppose we have the following:
\begin{enumerate}
\item A measurable function $\mathcal{L}:\mathcal{Z}\to[-\infty,\infty]$ that is  bounded below or is bounded above.
\item A distribution $P\in\mathcal{P}(\mathcal{Z})$ with $\mathcal{L}^-\in L^1(P)$.
\item $f\in\mathcal{F}_1(a,b)$ where $a\geq 0$.
\item  A  cost function, $c$, that satisfies $c(z,z)=0$  for all $z\in\mathcal{Z}$.
\end{enumerate}
Define $f^*(\pm\infty)\coloneqq\infty$. Then for $\epsilon>0$, $\kappa\geq 0$  we have
\begin{align}\label{eq:OT_f_div_DRO}
&\sup_{Q:D_f^c(Q\|P)\leq \epsilon}\{E_Q[\mathcal{L}]-\kappa D_f^c(Q\|P)\}\\
=&\inf_{\lambda>0,\rho\in\mathbb{R}}\{\lambda \epsilon+\rho+(\lambda+\kappa)E_{{P}}[f^*((\mathcal{L}^c_{\lambda+\kappa}-\rho)/(\lambda+\kappa))]\}\notag
\end{align}  
and the objective of the finite dimensional minimization is convex on  $(0,\infty)\times \mathbb{R}$.
\end{corollary}
{ A similar result can be proven for  R{\'e}nyi divergences by combining  Theorem \ref{thm:DRO_gen_P_gen_D} with Theorem \ref{thm:Renyi_cc} from Appendix \ref{app:additional_proofs}; see \eqref{eq:OT_reg_Renyi_DRO} .}

Finally, we  derive limiting formulas for the DRO problem, analogous to the interpolation results for $D^c$, Theorems \ref{thm:limit_D_c_to_C} and \ref{thm:limit_D_c_to_D}.  Though we don't use those theorems directly, the method of proof is similar and the conclusions align with what one expects in light of those results. We do require the more stringent assumptions that $\mathcal{Z}$ is compact and $\mathcal{L}$ is upper semicontinuous (USC),  which are often the case in practice. The proofs use many of the same techniques that were employed in Theorem \ref{thm:limit_D_c_to_C} and \ref{thm:limit_D_c_to_D} and so we defer details to Appendix \ref{supp:DRO_limit_thms}.
\begin{theorem}\label{thm:DRO_limit_r_0}
Suppose  the Polish space $\mathcal{Z}$ is compact and that we have the following:
\begin{enumerate}
\item An USC function $\mathcal{L}:\mathcal{Z}\to[-\infty,\infty)$.
\item A distribution $P\in\mathcal{P}(\mathcal{Z})$.
\item A pre-divergence $D$ such that $D(\cdot\|P)$ is LSC and $D(\mu_n\|P)\to 0$ implies $\mu_n\to P$ weakly.
\item A cost-function $c$.
\end{enumerate}
For $r>0$ define the cost functions $c_r=rc$. Then for $\epsilon>0$ we have
\begin{align}\label{eq:DRO_lim_0}
 \lim_{r\to 0^+}\sup_{Q:D^{c_r}(Q\|P)\leq r\epsilon}E_Q[\mathcal{L}]=\sup_{Q:C(P,Q)\leq \epsilon}E_Q[\mathcal{L}]\,.
\end{align}
If $\mathcal{L}_\theta:\mathcal{Z}\to[-\infty,\infty)$, $\theta\in\Theta$, is a family of USC functions then
\begin{align}\label{eq:DRO_min_r_0_limit}
\lim_{r\to 0^+}\inf_{\theta\in\Theta}\sup_{Q:D^{c_r}(Q\|P)\leq r\epsilon}E_Q[\mathcal{L}_\theta]=\inf_{\theta\in\Theta}\sup_{Q:C(P,Q)\leq \epsilon}E_Q[\mathcal{L}_\theta]\,.
\end{align}
\end{theorem}
\begin{theorem}\label{thm:DRO_limit_r_infty}
Suppose  the Polish space $\mathcal{Z}$ is compact and that we have the following:
\begin{enumerate}
\item An USC function $\mathcal{L}:\mathcal{Z}\to[-\infty,\infty)$.
\item A distribution $P\in\mathcal{P}(\mathcal{Z})$.
\item A pre-divergence $D$ such that $D(\cdot\|P)$ is LSC.
\item A cost-function $c$ that satisfies $c(z,\tilde z)=0$ iff $z=\tilde z$.
\end{enumerate}
For $r>0$ define the cost functions $c_r=rc$. Then for $\epsilon>0$ we have
\begin{align}
 \lim_{r\to \infty}\sup_{Q:D^{c_r}(Q\|P)\leq \epsilon}E_Q[\mathcal{L}]=\sup_{Q:D(Q\|P)\leq \epsilon}E_Q[\mathcal{L}]\,.
\end{align}
If $\mathcal{L}_\theta:\mathcal{Z}\to[-\infty,\infty)$, $\theta\in\Theta$, is a family of USC functions then
\begin{align}
 \lim_{r\to \infty}\inf_{\theta\in\Theta}\sup_{Q:D^{c_r}(Q\|P)\leq \epsilon}E_Q[\mathcal{L}_\theta]=\inf_{\theta\in\Theta}\sup_{Q:D(Q\|P)\leq \epsilon}E_Q[\mathcal{L}_\theta]\,.
\end{align}
\end{theorem}

\subsection{Pre-Divergences with Tractable Convex Conjugate}\label{sec:CC_general_D}

{ 
In this subsection, we present a method for constructing a large collection of pre-divergences, $D$,  that have computationally-tractable convex conjugates, expanding on the $f$-divergence and R{\'e}nyi cases from \eqref{eq:Df_Gibbs} and \eqref{eq:Renyi_cc} respectively. This general class of $D$'s can then be combined with Theorem \ref{thm:DRO_gen_P_gen_D} to obtain tractable reformulations of the corresponding OT-regularized  DRO problems.  To simplify the technicalities, we focus on the case where $\mathcal{Z}$ is  compact; this is often sufficient in practice. The non-compact case could similarly be studied using the theory of infimal convolutions, see, e.g., Theorem 2.3.10 in \cite{bot2009duality}.  We will equip $C(\mathcal{Z})$ with the topology from the supremum norm and $\mathcal{P}(\mathcal{Z})$ with the topology of weak convergence.
\begin{theorem}\label{thm:CC_general_D}
Let $\mathcal{Z}$ be a compact Polish space, and suppose for every $P\in\mathcal{P}(\mathcal{Z})$ we have $H_P: C(\mathcal{Z})\to(-\infty,\infty]$ that is convex, LSC,  and that satisfies $\inf_{\rho\in\mathbb{R}}\{\rho+H_P[-\rho]\}=0$  and  $H_P[g]\geq E_P[g]$ for all $g\in C(\mathcal{Z})$. Define $D_H:\mathcal{P}(\mathcal{Z})\times\mathcal{P}(\mathcal{Z})\to[0,\infty]$  by
\begin{align}\label{eq:DH_def}
D_H(Q\|P)\coloneqq\sup_{g\in C(\mathcal{Z})}\{E_Q[g] -H_P[g]\}\,.
\end{align}
Then $D_H$ is a pre-divergence, $D_H(\cdot\|P)$ is convex and LSC for all $P\in\mathcal{P}(\mathcal{Z})$, and the convex conjugate of $D_H(\cdot\|P)$ at $\phi\in C(\mathcal{Z})$ is given by
\begin{align}\label{eq:D_CC_general}
    \sup_{Q\in\mathcal{P}(\mathcal{Z}):D_H(Q\|P)<\infty}\{E_Q[\phi]-D_H(Q\|P)\}=\inf_{g\in C(\mathcal{Z})}\{\sup g+ H_P[\phi-g]\}\,.
\end{align}

If, in addition, we have the monotonicity property $g_1\leq g_2\implies H_P[g_1]\leq H_P[g_2]$ for all $P\in\mathcal{P}(\mathcal{Z})$, $g_1,g_2\in C(\mathcal{Z})$ then
\begin{align}\label{eq:D_CC_general_monotone}
    \sup_{Q\in\mathcal{P}(\mathcal{Z}):D_H(Q\|P)<\infty}\{E_Q[\phi]-D_H(Q\|P)\}=\inf_{\rho\in\mathbb{R}}\{\rho+ H_P[\phi-\rho]\}\,.
\end{align}
\end{theorem}
\begin{remark}
The $f$-divergence case \eqref{eq:Df_Gibbs} corresponds to 
$H_P[g]=  E_P[f^*(g)]$ and, formally, the R{\'e}nyi divergence case \eqref{eq:Renyi_cc} corresponds to $H_P[g]=\Lambda_\alpha^P[g]$ (to make this rigorous, it must be suitably modified to be LSC; we provide a rigorous proof of \eqref{eq:Renyi_cc} in Theorem \ref{thm:Renyi_cc}). Note that the specialized proof for $f$-divergences in  Corollary \ref{cor:CC_formula_OT_f_div} also applies to non-compact $\mathcal{Z}$.
\end{remark}
\begin{proof}
Non-negativity of $D_H$ follows from
\begin{align}
    D_H(Q\|P)\geq \sup_{\rho\in\mathbb{R}}\{-\rho-H_P[-\rho]\}=0\,.
\end{align}
Combined with the bound $H_P[g]\geq E_P[g]$ we see that $D_H(P\|P)=0$. Therefore $D_H$ is a pre-divergence.  The definition \eqref{eq:DH_def} clearly implies $D_H(\cdot\|P)$ is convex and LSC. Noting that $\mathcal{P}(\mathcal{Z})$ is convex and compact, $\{g\in C(\mathcal{Z}):H_P[g]<\infty\}$ is convex, and $(g,Q)\mapsto E_Q[\phi-g]+H_P[g]$ is continuous and linear in $Q$ for every $g$ and is convex and LSC in $g$ for every $Q$, we can apply Sion's Theorem \cite{Sion1958,komiya1988elementary} to obtain
\begin{align}
\sup_{Q:D_H(Q\|P)<\infty}\{E_Q[\phi]-D_H(Q\|P)\}
  =&\inf_{\substack{g\in C(\mathcal{Z}):\\H_P[g]<\infty}}\sup_{Q\in \mathcal{P}(\mathcal{Z})}\{E_Q[\phi-g]+H_P[g]\}\notag\\
  =&\inf_{\substack{g\in C(\mathcal{Z}):\\H_P[g]<\infty}}\{\sup\{\phi-g\}+H_P[g]\}\,.\notag
\end{align}
This implies \eqref{eq:D_CC_general}. In the monotone case, use $H_P[\phi-g]\geq H_P[\phi-\sup g]$ to conclude \eqref{eq:D_CC_general_monotone}.
\end{proof}
By combining  Theorem \ref{thm:DRO_gen_P_gen_D} with a family of functionals that satisfy the hypotheses of Theorem \ref{thm:CC_general_D} we arrive at the DRO problem reformulation
\begin{align}\label{eq:DH_DRO}
&\sup_{Q:D_H^c(Q\|P)\leq \epsilon}\{E_Q[\mathcal{L}]-\kappa D_H^c(Q\|P)\}\\
=&\inf_{\lambda>0,g\in C(\mathcal{Z})}\left\{\lambda \epsilon+\sup g+(\lambda+\kappa)H_P\left[(\lambda+\kappa)^{-1} (\mathcal{L}_{\lambda+\kappa}^c-g)\right]\right\}\,,\notag
\end{align} 
and if the $H_P$'s satisfy the above monotonicity property  then the infimum over $g\in C(\mathcal{Z})$ can be replaced by the infimum over $\rho\in\mathbb{R}$.
Theorem \ref{thm:CC_general_D} is thus a powerful tool for constructing a large class of $D$'s to which our DRO theory can be applied, and which lead to computationally tractable reformulations  of the  OT-regularized-divergence DRO problem.  In particular, note that given any such family   of $H_P$'s one can obtain another such family by adding  non-negative, LSC, convex penalty terms that vanish on constant functions, e.g., variance penalties $H_P[g]+\gamma\text{Var}_P[g]$, $\gamma>0$. }

\section{Conclusions}
\label{sec:conclusions}
 We introduced a new class of divergences for comparing probability distributions, the optimal-transport-regularized divergences, $D^c$, which are defined as an infimal convolution between an information divergence $D$ (such as KL) and an optimal-transport (OT) cost $C$. Using these new divergences within the distributionally robust optimization (DRO) framework, we proposed   the $ARMOR_D$ methods for enhancing the adversarial robustness of deep learning models.  The key innovation is the principled and dynamical manner in which the method combines transported adversarial samples with adversarial re-weighting of the samples via the information divergence.  In practice, the adversarial re-weighting focuses the optimization towards improving the performance on the most troublesome adversarial samples. We proved that these new tools have many attractive mathematical properties, making them well suited to applications in statistical learning, including adversarial training. As our method is based on the general DRO framework of $D^c$ neighborhoods, it can be used to augment any  empirical risk minimization problem. We demonstrated this flexibility by using $ARMOR_D$ to augment the UDR, TRADES, and MART methods and obtained improved performance in adversarially robust image recognition on CIFAR-10 and CIFAR-100. Specifically, augmenting with $ARMOR_D$ leads to 1.9\% and 2.1\% improvement  against  AutoAttack, a powerful ensemble of adversarial attacks, on CIFAR-10 and CIFAR-100 respectively.  Our experiments were  done using $ARMOR_D$ where $D$ was a {R{\'e}nyi divergence} or $f$-divergence, however the majority of the rigorous theoretical development we provided in Section \ref{sec:proofs} applies to a much more general class of $D$'s; { see Theorem \ref{thm:CC_general_D} for the construction of a large class of divergences, beyond the R{\'e}nyi and $f$-divergence cases, that lead to computationally tractable reformulations of the OT-regularized DRO problem \eqref{eq:DRO_Dc}. Investigating  the effectiveness of alternative OT-regularized divergences  at enhancing adversarial robustness, or other  learning tasks, is a promising direction for further exploration. Specifically, our framework can be applied  to  OT-regularizations   of variance-penalized divergences, which we intend to explore in future work. }

\appendix

\section{Additional Proofs}\label{app:additional_proofs}

\subsection{Convex Conjugate of the R{\'e}nyi Divergences}
Here we derive the    convex conjugate formula for R{\'e}nyi divergences that was stated in Eq.\,\eqref{eq:Renyi_cc}. The derivation builds on the work in \cite{birrell2023functionspace}, though to the best of the authors' knowledge the result presented here is new.
\begin{theorem}\label{thm:Renyi_cc}
Let $\mathcal{Z}$ be a compact metric space, $\alpha\in(0,1)\cup(1,\infty)$,  $P\in\mathcal{P}(\mathcal{Z})$, and $\phi\in C(\mathcal{Z})$.  Then
\begin{align}
\sup_{\eta\in \mathcal{P}(\mathcal{Z})}\{E_\eta[\phi]-R_\alpha(P\|\eta)\}
=& \inf_{\rho\in\mathbb{R}}\{\rho+\Lambda_\alpha^P[\phi-\rho]\}\,,
\end{align}
where
$\Lambda_\alpha^P:C(\mathcal{Z})\to(-\infty,\infty]$, 
\begin{align}\label{def:Lambda_alpha_P}
\Lambda_\alpha^P[g]\coloneqq-\left(\frac{1}{\alpha-1}\log\int |g|^{(\alpha-1)/\alpha} dP+\alpha^{-1}(\log\alpha +1)\right)1_{g<0}+\infty1_{g\not<0}\,.
\end{align}
\end{theorem}
\begin{proof}
$C(\mathcal{Z})$ is a  Banach space with dual $C(\mathcal{Z})^*=M(\mathcal{Z})$, the space of finite signed Borel measures on $\mathcal{Z}$ (see, e.g., Corollary 7.18 in \cite{folland2013real}).  As shown in Lemma 7.2 in \cite{birrell2023functionspace},  $F\coloneqq\Lambda_\alpha^P$ is convex and continuous on $\{g\in C(\mathcal{Z}):g<0\}$, an open set.  Define $G:C(\mathcal{Z})\to(-\infty,\infty]$, $G[g]\coloneqq\sup_{\mathcal{Z}}\{\phi-g\}$ and note that it is straightforward to show $G$ is convex.  We have $F[-1]<\infty$, $G[-1]<\infty$ and $F$ is continuous at $-1$, therefore Fenchel-Rockafellar duality (see, e.g., Theorem 4.4.3  in \cite{borwein2005techniques}) gives
\begin{align}
\inf_{g\in C(\mathcal{Z})}\{\Lambda_\alpha^P[g]+\sup_{\mathcal{Z}}\{\phi-g\}\}=\sup_{\mu\in M(\mathcal{Z})}\{-F^*[-\mu]-G^*[\mu]\}
\end{align}
or, after reparameterization,
\begin{align}
\inf_{g\in C(\mathcal{Z})}\{\Lambda_\alpha^P[\phi-g]+\sup_{\mathcal{Z}}g\}=\sup_{\mu\in M(\mathcal{Z})}\{-F^*[\mu]-G^*[-\mu]\}\,.
\end{align}
For $\mu\in M(\mathcal{Z})$ we have
\begin{align}\label{eq:G_star}
G^*[-\mu]=\sup_{g\in C(\mathcal{Z})}\left\{-\int gd\mu-\sup_{\mathcal{Z}}\{\phi-g\}\right\}=-\int\phi d\mu-\inf_{g\in C(\mathcal{Z})}\left\{\sup_{\mathcal{Z}}g-\int g d\mu\right\}\,.
\end{align}
We consider three cases:
\begin{enumerate}
\item If $\mu\in\mathcal{P}(\mathcal{Z})$ then it is clear that the infimum at the end  of \eqref{eq:G_star} equals $0$ and so $G^*[-\mu]=-E_\mu[\phi]$.
\item If $\mu(\mathcal{Z})\neq 1$ then $\inf_{g\in C(\mathcal{Z})}\{\sup g- \int gd\mu\}=-\infty$ and so $G^*[-\mu]=\infty$.
\item If $\mu$ is not positive then there exists $A$ with $\mu(A)<0$. Define the probability measure $\widetilde{P}=|\mu|/\|\mu\|$.  By Lusin's theorem, for all $\epsilon>0$ there exists a closed set $E_\epsilon$ with $\widetilde{P}(E_\epsilon^c)<\epsilon$ and there exists $g_\epsilon\in C(\mathcal{Z})$ such that $g_\epsilon|_{E_\epsilon}=1_A$, $0\leq g_\epsilon\leq 1$.  Therefore for $c\in (0,\infty)$ we have
\begin{align}
\inf_{g\in C(\mathcal{Z})}\left\{\sup_{\mathcal{Z}}g-\int g d\mu\right\}
\leq&\sup_{\mathcal{Z}}(-cg_\epsilon)-\int (-cg_\epsilon) d\mu\\
\leq& c\left(\int 1_{E_\epsilon}1_Ad\mu+\int 1_{E_{\epsilon}^c}g_\epsilon d\mu\right)\notag\\
\leq&c\left(\mu(A)-\int 1_{E_\epsilon^c}1_A d\mu+|\mu|(E_{\epsilon}^c)\right)\notag\\
\leq&c(\mu(A)+2\|\mu\|\epsilon)\,.\notag
\end{align}
This holds for all $c>0,\epsilon>0$ therefore we can let $\epsilon=|\mu(A)|/(4\|\mu\|)$ to get
\begin{align}
\inf_{g\in C(\mathcal{Z})}\left\{\sup_{\mathcal{Z}}g-\int g d\mu\right\}
\leq&-c|\mu(A)|/2
\end{align}
for all $c>0$. Sending $c\to\infty$ and noting that $|\mu(A)|>0$ we find
\begin{align}
    \inf_{g\in C(\mathcal{Z})}\left\{\sup_{\mathcal{Z}}g-\int g d\mu\right\}=-\infty\,.
\end{align} Hence $G^*[-\mu]=\infty$ in this case as well. 
\end{enumerate}
Combining these three cases, we have proven that 
\begin{align}
G^*[-\mu]=-\int\phi d\mu+\infty 1_{\mu\not\in\mathcal{P}(\mathcal{Z})}
\end{align}
Theorem   2.2 in \cite{birrell2023functionspace} implies  
 $(\Lambda_\alpha^P)^*[\mu]=R_\alpha(P\|\mu)$ for all    $\mu\in\mathcal{P}(\mathcal{Z})$. Therefore
\begin{align}
\inf_{g\in C(\mathcal{Z})}\{\Lambda_\alpha^P[\phi-g]+\sup_{\mathcal{Z}}g\}=&\sup_{\mu\in M(\mathcal{Z})}\{-F^*[\mu]-G^*[-\mu]\}\label{eq:Renyi_inf_g_lhs}\\
=&\sup_{\eta\in \mathcal{P}(\mathcal{Z})}\{-F^*[\eta]-G^*[-\eta]\}\notag\\
=&\sup_{\eta\in \mathcal{P}(\mathcal{Z})}\{E_\eta[\phi]-R_\alpha(P\|\eta)\}\,.\notag
\end{align}

Now we show that the infimum  over $g$ on the left-hand side of \eqref{eq:Renyi_inf_g_lhs} can be restricted to constant functions.  For $g\in C(\mathcal{Z})$ we can use monotonicity properties of the functions involved in $\Lambda_\alpha^P$ to obtain  
\begin{align}
&\Lambda_\alpha^P[\phi-g]\\
=&\infty1_{\phi-g\not<0}-\left(\frac{1}{\alpha-1}\log\int (g-\phi)^{(\alpha-1)/\alpha} dP+\alpha^{-1}(\log\alpha +1)\right)1_{\phi-g<0}\notag\\\geq&\begin{cases}
\infty \text{ if } \phi-g\not<0\\
-\left(\frac{1}{\alpha-1}\log\int (\sup g-\phi)^{(\alpha-1)/\alpha} dP+\alpha^{-1}(\log\alpha +1)\right) \text{ if }\phi-g<0
\end{cases}\notag\\
\geq&\begin{cases}
\infty \text{ if } \phi-\sup g\not<0\\
-\left(\frac{1}{\alpha-1}\log\int |\phi-\sup g|^{(\alpha-1)/\alpha} dP+\alpha^{-1}(\log\alpha +1)\right) \text{ if }\phi-\sup g<0
\end{cases}\notag\\
=&\Lambda_\alpha^P[\phi-\sup g]\,.\notag
\end{align}
Therefore
\begin{align}
\Lambda_\alpha^P[\phi-g]+\sup g
\geq &\Lambda_\alpha^P[\phi-\sup g]+\sup g\\
\geq & \inf_{\rho\in\mathbb{R}}\{\Lambda_\alpha^P[\phi-\rho]+\rho\}\,.\notag
\end{align}
This holds for all $g\in C(\mathcal{Z})$, hence
\begin{align}
&\inf_{g\in C(\mathcal{Z})}\{\Lambda_\alpha^P[\phi-g]+\sup g\}
\geq \inf_{\rho\in\mathbb{R}}\{\Lambda_\alpha^P[\phi-\rho]+\rho\}\,.
\end{align}
The reverse inequality is trivial and so we are done.
\end{proof}
\subsection{Weak Convergence and $f$-Divergences}\label{sec:setwise}
Here we show that $f$-divergences can be used to prove weak convergence of measures; this is needed in Theorem \ref{thm:D_c_div_property} as well as to apply many of the properties from Section \ref{app:properties} to OT-regularized $f$-divergences. In fact, we will prove the stronger setwise convergence property.  In this section we let $\mathcal{M}_b(\Omega)$ denote the set of bounded measurable real-valued functions on a measurable space $(\Omega,\mathcal{M})$.
\begin{definition}
Let $\{\mu_n\}_{n=1}^\infty,\mu$ be probability measures on the measurable space $(\Omega,\mathcal{M})$.  We say that $\mu_n\to \mu$ {\bf setwise} if $\lim_{n\to\infty}\mu_n(A)=\mu(A)$ for all $A\in\mathcal{M}$.
\end{definition}
First recall that setwise convergence implies convergence of integrals; we provide a simple proof of this fact.
\begin{lemma}\label{lemma:setwise_implies_integrals}
Let $(\Omega,\mathcal{M})$ be a measurable space and $\mu_n$, $\mu$ be probability measures on $\Omega$.  If $\mu_n\to\mu$ setwise then $\int \phi d\mu_n\to\int\phi d\mu$ for all $\phi\in\mathcal{M}_b(\Omega)$.
\end{lemma}
\begin{remark}
In particular, if $(\Omega,\mathcal{M})$ is a metric space with the Borel $\sigma$-algebra then this implies $\mu_n\to \mu$ weakly.
\end{remark}
\begin{proof}
Let $\phi\in\mathcal{M}_b(\Omega)$. Take a sequence of simple functions $\phi_j$ that converge uniformly to $\phi$ (see, e.g., Theorem 2.10 in \cite{folland2013real}).  With these we can compute
\begin{align}
|\int\phi d\mu_n-\int\phi d\mu|
\leq &|\int \phi d\mu_n-\int \phi_jd\mu_n|+|\int \phi_jd\mu_n-\int \phi_jd\mu|+|\int \phi_jd\mu-\int \phi d\mu|\notag\\
\leq & \|\phi-\phi_j\|_\infty(\mu_n(\Omega)+\mu(\Omega))+|\int \phi_jd\mu_n-\int \phi_jd\mu|\,.
\end{align}
 The fact that $\phi_j$ are simple and $\mu_n\to\mu$ setwise implies that $|\int \phi_jd\mu_n-\int \phi_jd\mu|\to 0$ as $n\to\infty$ for all $j$, hence
\begin{align}
\limsup_{n\to\infty}|\int \phi d\mu_n-\int \phi d\mu|\leq 2\|\phi-\phi_j\|_\infty
\end{align}
for all $j$.  Taking $j\to\infty$ completes the proof.
\end{proof}

We now prove that convergence of an $f$-divergence to zero implies setwise convergence under mild assumptions on $f$.
\begin{theorem}\label{thm:f_div_setwise}
Let $(\Omega,\mathcal{M})$ be a measurable space, $f\in\mathcal{F}_1(a,b)$, and define $ t_0\coloneqq f^\prime_+(1)$ (where $f^\prime_+$ denotes the right derivative of $f$, which exists due to the convexity of $f$).  Suppose $ t_0\in\{f^*<\infty\}^o$ (where $A^o$ denotes the interior of the set $A$) and that $f$ is strictly convex on a neighborhood of $1$.  If $P_n,P$ are probability measures on $\Omega$ and either $D_f(P_n\|P)\to 0$ or $D_f(P\|P_n)\to 0$ then $P_n\to P$ setwise.  

If $(\Omega,\mathcal{M})$ is a metric space with the Borel $\sigma$-algebra then we can further conclude $P_n\to P$ weakly.
\end{theorem}
\begin{proof}
Take any probability measures $Q_1,Q_2$ on $(\Omega,\mathcal{M})$ and $A\in\mathcal{M}$. For all $\epsilon>0$ we define $\phi_\epsilon= t_0+\epsilon 1_A$.  Then $\phi_\epsilon\in\mathcal{M}_b(\Omega)$, hence the variational representation of $f$-divergences (see Proposition B.1 in \cite{JMLR:v23:21-0100}) implies
\begin{align}
D_f(Q_1\|Q_2)\geq& E_{Q_1}[\phi_\epsilon]-E_{Q_2}[f^*(\phi_\epsilon)]\\
=& t_0+\epsilon Q_1(A)-E_{Q_2}[f^*( t_0+\epsilon 1_A)]\,.\notag
\end{align}
We have assumed that  $ t_0\in\{f^*<\infty\}^o$, hence there exists $\delta>0$ with $B_\delta( t_0)\subset \{f^*<\infty\}$.   Using properties of the Taylor expansion of convex functions (see \cite{LieseVajda}) along with the identities $f^*( t_0)= t_0$  and $(f^*)^\prime_+( t_0)=1$ (see Lemma A.9 in \cite{JMLR:v23:21-0100}) we can compute
\begin{align}
f^*(t)=&f^*( t_0)+(f^*)_+^\prime( t_0)(t- t_0)+R_{f^*}( t_0,t)\\
=&t+R_{f^*}( t_0,t)\notag\\
\leq &t+|t- t_0||(f^*)^\prime_+(t)-(f^*)^\prime_+( t_0)|\notag
\end{align}
for all $t\in B_\delta( t_0)$.  Letting $\epsilon<\delta$ we have $\text{range}( t_0+\epsilon 1_A)\subset B_\delta( t_0)$ and so
\begin{align}
f^*( t_0+\epsilon 1_A)\leq & t_0+\epsilon 1_A+| t_0+\epsilon 1_A- t_0||(f^*)^\prime_+( t_0+\epsilon 1_A)-(f^*)^\prime_+( t_0)|\\
=& t_0+\epsilon 1_A+\epsilon 1_A|(f^*)^\prime_+( t_0+\epsilon)-(f^*)^\prime_+( t_0)|\,.\notag
\end{align}
Hence
\begin{align}
D_f(Q_1\|Q_2)\geq& t_0+\epsilon Q_1(A)-E_{Q_2}[ t_0+\epsilon 1_A+\epsilon 1_A|(f^*)^\prime_+( t_0+\epsilon)-(f^*)^\prime_+( t_0)|]\\
=&\epsilon [Q_1(A)- Q_2(A)(1+|(f^*)^\prime_+( t_0+\epsilon)-(f^*)^\prime_+( t_0)|)]\,.\notag
\end{align}

Now let $P_n,P$ be probability measures on $\Omega$ and consider the following two cases.
\begin{enumerate}
\item Suppose $D_f(P_n,P)\to 0$.  Then letting $Q_1=P_n$ and $Q_2=P$ in the above we get
\begin{align}
0=&\limsup_nD_f(P_n\|P)\\
\geq&\epsilon [\limsup_nP_n(A)- P(A)(1+|(f^*)^\prime_+( t_0+\epsilon)-(f^*)^\prime_+( t_0)|)]\label{eq:bracket_term}
\end{align}
for all $\epsilon\in(0,\delta)$. If $\limsup_n P_n(A)>P(A)$ then by right-continuity of  $(f^*)^\prime_+$ (the right-derivative of a convex function), for $\epsilon$ small enough the term in brackets in \eqref{eq:bracket_term} is positive, which is a contradiction.  Therefore  $\limsup_n P_n(A)\leq P(A)$.    This holds for all $A\in\mathcal{M}$, hence for a given $A$ we can apply it to $A^c$ to get $\limsup_n P_n(A^c)\leq P(A^c)$, hence $\liminf_n P_n(A)\geq P(A)$.  Together these bounds imply $\lim_n P_n(A)=P(A)$ for all $A\in\mathcal{M}$, hence $P_n\to P$ setwise.
\item Suppose $D_f(P,P_n)\to 0$.  Letting $Q_1=P$ and $Q_2=P_n$ we have
\begin{align}
0=&\limsup_nD_f(P\|P_n)\\
\geq&\epsilon [P(A)-\liminf_n P_n(A)(1+|(f^*)^\prime_+( t_0+\epsilon)-(f^*)^\prime_+( t_0)|)]\,.\notag
\end{align}
If $P(A)>\liminf_n P_n(A)$ then for $\epsilon$ sufficiently small we again find the term in brackets to be positive, which is a contradiction.  Hence $P(A)\leq \liminf_n P_n(A)$ for all $A\in\mathcal{M}$.  Applying this to $A^c$ and combining the results gives $\lim_{n\to\infty}P_n(A)=P(A)$ for all $A\in\mathcal{M}$. Hence  we find that $P_n\to P$ setwise in this case as well.
\end{enumerate}

If $(\Omega,\mathcal{M})$ is a metric space with the Borel $\sigma$-algebra then we can further conclude $P_n\to P$ weakly by using Lemma \ref{lemma:setwise_implies_integrals}.
\end{proof}

\subsection{Interchangeability Result}\label{supp:interchangability}
In this section we prove a required interchangeability result. The proof mimics the strategy used for the more general result stated in \cite{2022arXiv220500362Z}; for completeness we provide a proof of the version that is used in the main text. Below we will use the notation $\mathcal{M}_\mu$ for the completion of a $\sigma$-algebra, $\mathcal{M}$, with respect to a measure $\mu$, we will denote the completion of $\mu$ by $\overline{\mu}$, and $\mathcal{M}_*$ will denote the $\sigma$-algebra of universally measurable sets (with respect to $\mathcal{M}$).  
\begin{lemma}[Interchangeability]
Let $\mu\in\mathcal{P}(\mathcal{Z})$ and  $\phi:\mathcal{Z}\times \mathcal{Z}\to[-\infty,\infty]$ be  measurable.   Then $z\mapsto\sup_{\tilde z\in \mathcal{Z}}\phi(z,\tilde z)$ is a $\mathcal{B}(\mathcal{Z})_*$-measurable function  and
\begin{align}\label{supp_eq:interchangability}
\sup_{\pi\in\mathcal{P}(\mathcal{Z}\times \mathcal{Z}):\pi_1=\mu} E_\pi[\phi]=\int \sup_{\tilde z\in \mathcal{Z}}\phi(z,\tilde z)\overline{\mu}(dz)\,.
\end{align}
\end{lemma}
\begin{remark}
In \eqref{supp_eq:interchangability} we use the convention $\infty-\infty\coloneqq-\infty$ to ensure all integrals therein are defined, though when using this result in the proof of Theorem \ref{thm:Dc_convex_conjugate_general_P} in the main text we will have further assumptions that guarantee all integrals are defined without relying on any such convention. 
\end{remark}
\begin{proof}
Define $\Phi=\sup_{\tilde z\in\mathcal{Z}}\phi(\cdot,y)$.  For $a\in\mathbb{R}$ we have
\begin{align}
&\{z:\Phi(z)>a\}=\{z:\exists \tilde z, \phi(z,\tilde z)>a\}\,,
\end{align}
which is the projection of the measurable set $\phi^{-1}((a,\infty])$ onto its first component. Therefore the measurable projection theorem (see, e.g., Proposition 8.4.4 in  \cite{cohn2013measure}) implies $\{z:\Phi(z)>a\}$ is $\mathcal{B}(\mathcal{Z})_*$-measurable. The rays $(a,\infty]$ for $a\in\mathbb{R}$ generate the $\sigma$-algebra on $[-\infty,\infty]$, hence $\Phi$ is universally measurable as claimed.

To prove \eqref{supp_eq:interchangability}, first suppose $\int \Phi^-d\overline{\mu}<\infty$, where $\Phi^-$ denotes the negative part of $\Phi$. Define $\Phi_n=\min\{n,\Phi-1/n\}$ and note that $\min\{0,\Phi-1\}\leq \Phi_n<\Phi$,  and $\Phi_n\nearrow \Phi$. The $\Phi_n$ are  universally measurable,  therefore $C_n\coloneqq\{(z,\tilde z)\in\mathcal{Z}\times \mathcal{Z}:\phi(z,\tilde z)>\Phi_n(z)\}$ are $\mathcal{B}(\mathcal{Z})_\mu\bigotimes\mathcal{B}(\mathcal{Z})$-measurable. For every $z\in\mathcal{Z}$ we have $\Phi_n(z)<\Phi(z)=\sup_{\tilde z\in\mathcal{Z}}\phi(z,\tilde z)$, hence there exists $y\in \mathcal{Z}$ such that $(z,\tilde z)\in C_n$. Therefore the projection of $C_n$ onto its first component equals $\mathcal{Z}$. The measurable selection theorem,  Corollary 8.5.4 in \cite{cohn2013measure}, then implies that there exists $T_n:\mathcal{Z}\to \mathcal{Z}$ that is $((\mathcal{B}(\mathcal{Z})_\mu)_*,\mathcal{B}(\mathcal{Z}))$-measurable  such that the graph of $T_n$ is contained in $C_n$. Using the result from Ex.~8.4.2(b) in  \cite{cohn2013measure} we have  $({\mathcal{B}(\mathcal{Z})}_\mu)_*={\mathcal{B}(\mathcal{Z})}_\mu$, therefore  $T_n$ is $(\mathcal{B}(\mathcal{Z})_\mu,\mathcal{B}(\mathcal{Z}))$-measurable.  The map $\psi:z\mapsto (z,T_n(z))$ is $({\mathcal{B}(\mathcal{Z})}_\mu,\mathcal{B}(\mathcal{Z})\bigotimes \mathcal{B}(\mathcal{Z}))$-measurable and the pushforward measure $\psi_\#\overline{\mu}\in P(\mathcal{Z}\times\mathcal{Z})$ satisfies $(\psi_\#\overline{\mu})_1=\mu$, therefore 
\begin{align}
&\sup_{\pi\in\mathcal{P}(\mathcal{Z}\times \mathcal{Z}):\pi_1=\mu} E_\pi[\phi]\geq E_{\psi_\#\overline{\mu}}[\phi]
=\int \phi(z,T_n(z)) \overline{\mu}(dz)\geq \int \Phi_nd\overline{\mu}\,,
\end{align}
where in the last inequality we used the fact that the graph of $T_n$ is contained in $C_n$.  We have the  lower bound $\Phi_n\geq -\Phi^--1\in L^1(\overline{\mu})$ and therefore we can use the monotone convergence theorem to obtain
\begin{align}\label{eq:interchange_ineq1}
&\sup_{\pi\in\mathcal{P}(\mathcal{Z}\times \mathcal{Z}):\pi_1=\mu} E_\pi[\phi]\geq \lim_{n\to\infty}\int \Phi_nd\overline{\mu}  = \int \Phi d\overline{\mu}\,.
\end{align}
This also trivially holds if $\int \Phi^-d\overline{\mu}=\infty$ due to our convention $\infty-\infty\coloneqq-\infty$. The reverse inequality follows easily from the bound $\Phi(z)\geq \phi(z,\tilde z)$ for all $z$ and $\tilde z$, together with the fact that  there exists a $\mathcal{B}(\mathcal{Z})$-measurable function that equals  $\Phi$  $\overline{\mu}$-a.s.   
\end{proof}

\subsection{OT-Regularized Divergence Interpolation Theorem}\label{supp:interp_thm}

\begin{theorem}[Interpolation] 
Let $D$ be a pre-divergence, $c$ be a cost function, and $\mu,\nu\in\mathcal{P}(\mathcal{Z})$ that satisfy the following.
\begin{enumerate}
\item The mapping $P\mapsto D(P\|\mu)$ is LSC and has compact sublevel sets.
\item  $c(z,\tilde z)=0$ iff $z=\tilde z$.
\end{enumerate}
  For $r>0$ define the cost function $c_r=rc$.   Then
\begin{align}
\lim_{r\to \infty}D^{c_r}(\nu\|\mu)=D(\nu\|\mu)\,\, \text{ for all } \mu,\nu\in \mathcal{P}(\mathcal{Z}).
\end{align}
\end{theorem}
\begin{proof}
 From the definitions we have
\begin{align}
D^{c_r}(\nu\|\mu)= \inf_{\eta\in\mathcal{P}(\mathcal{Z})}\{D(\eta\|\mu)+rC(\eta,\nu)\}\leq D(\nu\|\mu)
\end{align}
and the left-hand side is non-decreasing in $r$. Therefore for $r_n\nearrow \infty$ we have
\begin{align}
\lim_{n\to\infty}D^{c_{r_n}}(\nu\|\mu)=\sup_{n}D^{c_{r_n}}(\nu\|\mu)\leq D(\nu\|\mu)\,.
\end{align}
We now show that assuming this inequality is strict leads lead to a contradiction, thus completing the proof.  Suppose that $\sup_{n}D^{c_{r_n}}(\nu\|\mu)< D(\nu\|\mu)$. Theorem \ref{thm:inf_conv_solution} implies the existence of $\eta_{*,n}$ such that
\begin{align}
D^{c_{r_n}}(\nu\|\mu)=D(\eta_{*,n}\|\mu)+r_nC(\eta_{*,n},\nu)\,.
\end{align}
In particular, $\sup_nD(\eta_{*,n}\|\mu)\leq \sup_nD^{c_{r_n}}(\nu\|\mu)<\infty$ and therefore $\eta_{*,n}$ all lie in a sublevel set of $D(\cdot\|\mu)$, which is compact.  Hence there exists a weakly convergent subsequence $\eta_{*,n_j}\to \eta_*$.  Next we show that $\eta_*=\nu$.  To do this, note that
\begin{align}
\infty>\sup_nD^{c_{r_n}}(\nu\|\mu)\geq \sup_n r_n C(\eta_{*,n},\nu)
\end{align}
and therefore $\lim_n C(\eta_{*,n},\nu)=0$.  Lower semicontinuity  implies $0=\lim_j C(\eta_{*,n_j},\nu)\geq C(\eta_*,\nu)\geq 0$  and so $C(\eta_*,\nu)=0$.  The cost function has the property $c(z,\tilde z)=0$ iff $z=\tilde z$, hence we can conclude that $\eta_*=\nu$.  To complete the proof we can use the lower semicontinuity of $D(\cdot\|\mu)$ to compute
\begin{align}
D(\nu\|\mu)>\sup_{n}D^{c_{r_n}}(\nu\|\mu)\geq \liminf_j  D(\eta_{*,n_j}\|\mu)\geq D(\eta_*\|\mu)= D(\nu\|\mu)\,.
\end{align}
This is a contradiction and so the proof is complete.
\end{proof}

\subsection{Convex Conjugate of OT-Regularized $f$-Divergence}
\begin{corollary}\label{cor:CC_formula_OT_f_div_supp}
Suppose we have the following:
\begin{enumerate}
\item A measurable function $\mathcal{L}:\mathcal{Z}\to[-\infty,\infty]$ that is  bounded below or is bounded above.
\item A distribution $P\in\mathcal{P}(\mathcal{Z})$ with $\mathcal{L}^-\in L^1(P)$, where $\mathcal{L}^-$ denotes the negative part of $\mathcal{L}$.
\item $f\in\mathcal{F}_1(a,b)$ with $a\geq 0$.
\item  A  cost function, $c$, that satisfies $c(z,z)=0$ for all $z\in\mathcal{Z}$.
\end{enumerate}
Then for $\lambda>0$ we have
\begin{align}
\sup_{Q\in\mathcal{P}(\mathcal{Z}):D_f^c(Q\|P)<\infty}\{E_Q[\mathcal{L}]-\lambda D_f^c(Q\|P)\}=\lambda\inf_{\rho\in\mathbb{R}}\{\rho+E_{{P}}[f^*(\lambda^{-1} \mathcal{L}^c_\lambda-\rho)]\}\,,
\end{align}
where the definition of $f^*$ is extended by $f^*(\pm\infty)\coloneqq\infty$.
\end{corollary}
\begin{proof} 
$D_f$ is a pre-divergence and $D_f(\cdot\|P)$ is convex, hence Theorem \ref{thm:Dc_convex_conjugate_general_P}  gives
\begin{align}\label{eq:CC_formula_OT_f_div_intermediate}
\sup_{Q\in\mathcal{P}(\mathcal{Z}):D_f^c(Q\|P)<\infty}\{E_Q[\mathcal{L}]-\lambda D_f^c(Q\|P)\}=\lambda\sup_{Q\in\mathcal{P}(\mathcal{Z}):D_f(Q\|P)<\infty}\{E_{\overline{Q}}[ \lambda^{-1}\mathcal{L}_\lambda^c]- D_f(Q\|P)\}
\end{align}
for all $\lambda>0$.  We have $\mathcal{L}^c_\lambda(z)\geq \mathcal{L}(z)$ and so $(\mathcal{L}^c_\lambda)^-\leq \mathcal{L}^-\in L^1(P)$.  Therefore $(\mathcal{L}^c_\lambda)^-\in L^1(P)$ and we can employ the Gibbs variational principle for $f$-divergences (see  Theorem 4.2 in \cite{BenTal2007}) to compute
\begin{align}\label{eq:Gibbs_unbounded}
\sup_{Q: D_f(Q\|P)<\infty}\{E_{\overline{Q}}[\lambda^{-1}\mathcal{L}^c_\lambda]-D_f(Q\|P)\}=\inf_{\rho\in\mathbb{R}}\{\rho+E_{\overline{P}}[f^*(\lambda^{-1}\mathcal{L}^c_\lambda-\rho)]\}\,.
\end{align}
We revert to explicit completion notation here to clarify a technical point. Theorem 4.2 from \cite{BenTal2007} assumes measurability of the integrand and not universal measurability.  However one can easily prove that \eqref{eq:Gibbs_unbounded} still follows by first replacing $\mathcal{L}^c_\lambda$ with a $\mathcal{B}(\mathcal{Z})$-measurable function that agrees with it $\overline{P}$-a.s. and then using the fact that $D_f(Q\|P)<\infty$ implies $Q\ll P$; see Definition \eqref{def:f_div}. Also, the result in \cite{BenTal2007} assumes the integrand on the left-hand side of \eqref{eq:Gibbs_unbounded} is in $L^1(P)$ but the case where the positive part is not integrable is easily checked to yield infinity on both sides of the identity. Combining \eqref{eq:Gibbs_unbounded} and \eqref{eq:CC_formula_OT_f_div_intermediate} completes the proof.
\end{proof}

\subsection{A Tractable Reformulation of the OT-Regularized $f$-Divergence DRO Problem}\label{supp:reformulation_OT_f_DRO}

\begin{corollary}\label{cor:OT_f_div_DRO_app}
Suppose we have the following:
\begin{enumerate}
\item A measurable function $\mathcal{L}:\mathcal{Z}\to[-\infty,\infty]$ that is  bounded below or is bounded above.
\item A distribution $P\in\mathcal{P}(\mathcal{Z})$ with $\mathcal{L}^-\in L^1(P)$.
\item $f\in\mathcal{F}_1(a,b)$ where $a\geq 0$.
\item  A  cost function, $c$, that satisfies $c(z,z)=0$  for all $z\in\mathcal{Z}$.
\end{enumerate}
Define $f^*(\pm\infty)\coloneqq\infty$. Then for $\epsilon>0$, $\kappa\geq 0$  we have
\begin{align}\label{eq:OT_f_div_DRO_app}
&\sup_{Q:D_f^c(Q\|P)\leq \epsilon}\{E_Q[\mathcal{L}]-\kappa D_f^c(Q\|P)\}\\
=&\inf_{\lambda>0,\rho\in\mathbb{R}}\{\lambda \epsilon+\rho+(\lambda+\kappa)E_{{P}}[f^*((\mathcal{L}^c_{\lambda+\kappa}-\rho)/(\lambda+\kappa))]\}\notag
\end{align}  
and the objective of the finite dimensional minimization is convex on  $(0,\infty)\times \mathbb{R}$.
\end{corollary}
\begin{proof}
Equation \eqref{eq:OT_f_div_DRO_app} follows from applying Theorem \ref{thm:DRO_gen_P_gen_D} to $D=D_f$ and then evaluating the convex conjugate of $D_f(\cdot\|P)$ by the same method as in Corollary \ref{cor:CC_formula_OT_f_div}. To prove convexity of the objective function, first note that for all $z$ the maps $h_z(t)\coloneqq \sup_{\tilde z}\{t\mathcal{L}(\tilde z)-c(z,\tilde z)\}$ are convex in $t\in(0,\infty)$ and either $h_z>-\infty$ or $h_z(t)=-\infty$ for all $t$.  $f^*$ is convex and it is straightforward to check that $a\geq 0$ implies $f^*$ is non-decreasing on $(-\infty,\infty]$. These facts together imply that  $(t,\rho)\mapsto f^*(h_z(t)-\rho)$ are convex on $(0,\infty)\times\mathbb{R}$ for all $z$. Linearity of the expectation then implies $H(t,\rho)\coloneqq E_P[f^*(h_z(t)-\rho)]$  is convex (note that $f^*(s)\geq s$ and so the assumptions on $\mathcal{L}$ imply that $H(t,\rho)>-\infty$).  Therefore  the perspective of $H$, given by $(\lambda,t,\rho)\mapsto \lambda  H(t/\lambda,\rho/\lambda)$, is convex on $(0,\infty)\times (0,\infty)\times\mathbb{R}$. Composing with the affine map $(\lambda,\rho)\mapsto (\lambda+\kappa,1,\rho)$ and adding the linear term $\lambda\epsilon+\rho$ results in a convex function on $(0,\infty)\times\mathbb{R}$. This completes the proof.
\end{proof}

\subsection{OT-Regularized-Divergence DRO Interpolation Theorems}\label{supp:DRO_limit_thms}
\begin{theorem}\label{supp_thm:DRO_limit_r_0}
Suppose  the Polish space $\mathcal{Z}$ is compact and that we have the following:
\begin{enumerate}
\item An USC function $\mathcal{L}:\mathcal{Z}\to[-\infty,\infty)$.
\item A distribution $P\in\mathcal{P}(\mathcal{Z})$.
\item A pre-divergence $D$ such that $D(\cdot\|P)$ is LSC and $D(\mu_n\|P)\to 0$ implies $\mu_n\to P$ weakly.
\item A cost-function $c$.
\end{enumerate}
For $r>0$ define the cost functions $c_r=rc$. Then for $\epsilon>0$ we have
\begin{align}\label{supp_eq:DRO_lim_0}
 \lim_{r\to 0^+}\sup_{Q:D^{c_r}(Q\|P)\leq r\epsilon}E_Q[\mathcal{L}]=\sup_{Q:C(P,Q)\leq \epsilon}E_Q[\mathcal{L}]\,.
\end{align}
If $\mathcal{L}_\theta:\mathcal{Z}\to[-\infty,\infty)$, $\theta\in\Theta$, is a family of USC functions then
\begin{align}\label{supp_eq:DRO_min_r_0_limit}
\lim_{r\to 0^+}\inf_{\theta\in\Theta}\sup_{Q:D^{c_r}(Q\|P)\leq r\epsilon}E_Q[\mathcal{L}_\theta]=\inf_{\theta\in\Theta}\sup_{Q:C(P,Q)\leq \epsilon}E_Q[\mathcal{L}_\theta]\,.
\end{align}
\end{theorem}
\begin{proof}
Compactness of $\mathcal{Z}$ and upper semicontinuity of $\mathcal{L}$ implies that $\mathcal{L}$ has a maximizer and hence $\mathcal{L}$ is bounded above. In particular, the expectations in \eqref{supp_eq:DRO_lim_0} are all well-defined. The bound $D^{c_r}(Q\|P)\leq rC(P,Q)$ implies 
\begin{align}\label{eq:DRO_limit_r_0_lb}
\sup_{Q:C(P,Q)\leq \epsilon}E_Q[\mathcal{L}]\leq \sup_{Q:D^{c_r}(Q\|P)\leq r\epsilon}E_Q[\mathcal{L}]
\end{align}
for all $r>0$. Define $K_r\coloneqq \{Q:D^{c_r}(Q\|P)\leq r\epsilon\}$ and note that $r_2\leq r_1$ implies $K_{r_2}\subset K_{r_1}$, hence the right-hand side of \eqref{eq:DRO_limit_r_0_lb} is non-decreasing in $r$.  $D(\cdot\|P)$ is LSC, therefore it has closed sublevel sets. $\mathcal{Z}$ is a compact Polish space, hence  $\mathcal{P}(\mathcal{Z})$ is compact (see, e.g., page 117 of \cite{bogachev2018weak}). Therefore the sublevel sets of $D(\cdot\|P)$ are also compact. Theorem  \ref{thm:Dc_LSC} then implies $D^{c_r}(\cdot\|P)$ is LSC for all $r$ and so the $K_r$ are closed sets, hence also compact. $\mathcal{L}$ is USC and bounded above, therefore the Portmanteau theorem implies that $Q\mapsto E_Q[\mathcal{L}]$ is USC, hence it achieves its maximum on $K_r$, i.e., there exists $Q_r\in K_r$ such that
\begin{align}
\sup_{Q\in K_r}E_Q[\mathcal{L}]=E_{Q_r}[\mathcal{L}].  
\end{align}
 Take $r_n\searrow 0^+$.   Compactness of $\mathcal{P}(\mathcal{Z})$ implies the existence of a weakly convergent subsequence $Q_j\coloneqq Q_{r_{n_j}}\to Q_*$.  Upper semicontinuity then implies $\limsup_jE_{Q_j}[\mathcal{L}]\leq E_{Q_*}[\mathcal{L}]$. By Theorem \ref{thm:inf_conv_solution} there exist $\eta_{*,j}\in\mathcal{P}(\mathcal{Z})$ such that
\begin{align}
\epsilon\geq \frac{1}{r_{n_j}}D^{c_{r_{n_j}}}(Q_j\|P)=\frac{1}{r_{n_j}}D(\eta_{*,j}\|P)+C(\eta_{*,j},Q_j)\,.
\end{align}
In particular we see that $\lim_{j\to\infty}D(\eta_{*,j}\|P)=0$.  The assumptions on $D$ therefore imply  that $\eta_{*,j}\to P$ weakly. Now we can use  lower semicontinuity of $C$ to compute 
\begin{align}
C(P,Q_*)\leq \liminf_j C(\eta_{*,j},Q_j)\leq \liminf_j\frac{1}{r_{n_j}}D^{c_{r_{n_j}}}(Q_j\|P)\leq \epsilon\,.
\end{align}
Therefore $Q^*\in\{Q:C(P,Q)\leq \epsilon\}$. Putting these pieces together we find
\begin{align}
\sup_{Q:C(P,Q)\leq \epsilon}E_Q[\mathcal{L}]\geq& E_{Q^*}[\mathcal{L}]\geq\limsup_j E_{Q_{r_{n_j}}}[\mathcal{L}]=\limsup_j \sup_{Q\in K_{r_{n_j}}}E_Q[\mathcal{L}]\\
\geq &\inf_{r>0}\sup_{Q:D^{c_r}(Q\|P)\leq r\epsilon}E_Q[\mathcal{L}]=\lim_{r\to 0^+}\sup_{Q:D^{c_r}(Q\|P)\leq r\epsilon}E_Q[\mathcal{L}]\,,\notag
\end{align}
where the last equality follows from the fact that the right-hand side is non-decreasing in $r$. Combining this inequality with \eqref{eq:DRO_limit_r_0_lb}  completes the proof of \eqref{supp_eq:DRO_lim_0}. If one has a family of USC functions $\mathcal{L}_\theta$ then apply \eqref{supp_eq:DRO_lim_0} for each $\theta$ and note  that the limit as $r\to 0^+$ is an infimum, hence one can commute the infimum over $\theta$ with the limit $r\to 0^+$ to obtain \eqref{supp_eq:DRO_min_r_0_limit}.
\end{proof}
\begin{theorem}\label{supp_thm:DRO_limit_r_infty}
Suppose  the Polish space $\mathcal{Z}$ is compact and that we have the following:
\begin{enumerate}
\item An USC function $\mathcal{L}:\mathcal{Z}\to[-\infty,\infty)$.
\item A distribution $P\in\mathcal{P}(\mathcal{Z})$.
\item A pre-divergence $D$ such that $D(\cdot\|P)$ is LSC.
\item A cost-function $c$ that satisfies $c(z,\tilde z)=0$ iff $z=\tilde z$.
\end{enumerate}
For $r>0$ define the cost functions $c_r=rc$. Then for $\epsilon>0$ we have
\begin{align}
 \lim_{r\to \infty}\sup_{Q:D^{c_r}(Q\|P)\leq \epsilon}E_Q[\mathcal{L}]=\sup_{Q:D(Q\|P)\leq \epsilon}E_Q[\mathcal{L}]\,.
\end{align}
If $\mathcal{L}_\theta:\mathcal{Z}\to[-\infty,\infty)$, $\theta\in\Theta$, is a family of USC functions then
\begin{align}
 \lim_{r\to \infty}\inf_{\theta\in\Theta}\sup_{Q:D^{c_r}(Q\|P)\leq \epsilon}E_Q[\mathcal{L}_\theta]=\inf_{\theta\in\Theta}\sup_{Q:D(Q\|P)\leq \epsilon}E_Q[\mathcal{L}_\theta]\,.
\end{align}
\end{theorem}
The proof of Theorem \ref{supp_thm:DRO_limit_r_infty} is similar to that of Theorem \ref{supp_thm:DRO_limit_r_0}, with slight differences that are motivated by the proof of Theorem \ref{thm:limit_D_c_to_D}; we omit the details.

\section{Interpreting the Outer Minimizer: Adversarial Sample Weights}\label{app:reweighting}

In this section we derive the (formal) solution  \eqref{eq:OT_reg_Df_DRO_optim_main}  to the optimization problem \eqref{eq:OT_reg_Df_DRO} that was presented in Section \ref{sec:reweighting}. We work under the assumptions that  $\mathcal{Z}=\mathbb{R}^d$, exact optimizers exist, and all functions are sufficiently smooth. 

Begin by letting $\tilde z_i(\lambda)$ be the solution to the inner maximizer \eqref{eq:OT_regularized_loss_main} with $z=z_i$  as a function of $\lambda$ and let $\lambda_*$ and $\rho_*$ be the optimal scaling and shift parameters  for the outer minimizer at a fixed $\theta$ (we suppress the $\theta$-dependence of $\tilde z_i$, $\lambda_*$, and $\rho_*$ in the notation). Taking the gradient of the objective function for the inner maximizer \eqref{eq:OT_regularized_loss_main} with respect to $\tilde z$ and evaluating at the optimizer $\tilde z_i(\lambda)$, we find
\begin{align}\label{eq:dy_OT_reg_loss}
\nabla_{\tilde z}\mathcal{L}_\theta(\tilde z_i(\lambda))-\lambda\nabla_{\tilde z} c(z_i,\tilde  z_i(\lambda))=0
\end{align}
for all $\lambda$.  Differentiating the objective function in \eqref{eq:OT_reg_Df_DRO_optim} with respect to $\rho$ we find
\begin{align}
    &\partial_\rho|_{\rho=\rho_*}(\epsilon\lambda_*+ \rho+\lambda_* \frac{1}{n}\sum_{i=1}^nf^*((\mathcal{L}_{\theta,\lambda_*}^{c}(z_i)-\rho)/\lambda_*))\\
    =&1+\lambda_*\frac{1}{n}\sum_{i=1}^n (f^*)^\prime((\mathcal{L}^c_{\theta,\lambda_*}(z_i)-\rho_*)/\lambda_*)(-\lambda_*^{-1})=0\,,\notag
\end{align}
i.e.,
\begin{align}\label{eq:new_weight_sum}
\frac{1}{n}\sum_{i=1}^n (f^*)^\prime((\mathcal{L}^c_{\theta,\lambda_*}(z_i)-\rho_*)/\lambda_*)=1\,.
\end{align}
In particular, this implies that the $p_{*,i}$'s, defined by
\begin{align}\label{eq:new_weights}
   p_{*,i}\coloneqq\frac{1}{n}(f^*)^\prime((\mathcal{L}^c_{\theta,\lambda_*}(z_i)-\rho_*)/\lambda_*)\,,
\end{align}
sum to $1$. We next differentiate the objective function with respect to $\lambda$ to obtain
\begin{align}\label{eq:ARMOR_parial_lambda}
&\epsilon+\frac{1}{n}\sum_i f^*((\mathcal{L}^c_{\theta,\lambda_*}(z_i)-\rho_*)/\lambda_*)\\
&+\lambda\frac{1}{n}\sum_i (f^*)^\prime((\mathcal{L}^c_{\theta,\lambda_*}(z_i)-\rho_*)/\lambda_*)(\partial_\lambda|_{\lambda=\lambda_*}(\lambda^{-1}\mathcal{L}^c_{\theta,\lambda}(z_i))+\rho_*/\lambda_*^2)=0\,,\notag
\end{align}
where we can use \eqref{eq:dy_OT_reg_loss} to simplify
\begin{align}\label{eq:inner_max_partial_lambda}
    \partial_\lambda(\lambda^{-1}\mathcal{L}^c_{\theta,\lambda}(z_i))=&-\lambda^{-2}\mathcal{L}_\theta({\tilde z}_i(\lambda))+(\lambda^{-1}\nabla_{\tilde z}\mathcal{L}_\theta({\tilde z}_i(\lambda))-\nabla_{\tilde z}c(z_i,{\tilde z}_i(\lambda)))\cdot {\tilde z}_i^\prime(\lambda)\\
    =&-\lambda^{-2}\mathcal{L}_\theta({\tilde z}_i(\lambda))\,.\notag
\end{align}
Combining \eqref{eq:new_weight_sum}, \eqref{eq:ARMOR_parial_lambda}, and \eqref{eq:inner_max_partial_lambda} we can compute
\begin{align}
    &\epsilon\lambda_*+\rho_*+\lambda_*\frac{1}{n}\sum_if^*((\mathcal{L}^c_{\theta,\lambda_*}(z_i)-\rho_*)/\lambda_*)\\
    =&\frac{1}{n}\sum_i (f^*)^\prime((\mathcal{L}^c_{\theta,\lambda_*}(z_i)-\rho_*)/\lambda_*)(\mathcal{L}_\theta({\tilde z}_i(\lambda_*))-\rho_*)-\lambda_*\frac{1}{n}\sum_i f^*((\mathcal{L}^c_{\theta,\lambda_*}(z_i)-\rho_*)/\lambda_*)\notag\\
    &+\rho_*+\lambda_*\frac{1}{n}\sum_i f^*((\mathcal{L}^c_{\theta,\lambda_*}(z_i)-\rho_*)/\lambda_*)\notag\\
    =&\frac{1}{n}\sum_i (f^*)^\prime((\mathcal{L}^c_{\theta,\lambda_*}(z_i)-\rho_*)/\lambda_*)\mathcal{L}_\theta({\tilde z}_i(\lambda_*))\,.\notag
\end{align}
Recalling the definition of $\lambda_*$ and $\rho_*$, this implies
\begin{align}\label{eq:OT_reg_Df_DRO_optim}
\inf_{\lambda>0,\rho\in\mathbb{R}}\left\{ \epsilon\lambda+ \rho+\lambda \frac{1}{n}\sum_{i=1}^nf^*((\mathcal{L}_{\theta,\lambda}^{c}(z_i)-\rho)/\lambda)\right\}=E_{Q_{*,\theta}}[\mathcal{L}_\theta]\,,
\end{align}
where the optimal adversarial distribution is
\begin{align}\label{eq:Q_star_app}
    Q_{*,\theta}\coloneqq\sum_{i=1}^n\frac{1}{n}(f^*)^\prime((\mathcal{L}^c_{\theta,\lambda_*}(z_i)-\rho_*)/\lambda_*) \delta_{{\tilde z}_i(\lambda_*)}\,.
\end{align}
This is the equality claimed above in \eqref{eq:OT_reg_Df_DRO_optim}. See Section \ref{sec:reweighting} for a discussion of the implications of this formula.

\subsection{Relation to OT-Based Adversarial Training} 
The robust-optimization adversarial training method \eqref{eq:classical_robust_optimization}  is a special case of OT-DRO, as was noted previously in \cite{regniez2021distributional,bui_UDR_2022unified}, with a specific choice of cost function.  The $ARMOR_D$ methods, which allow for general OT cost, are therefore  generalizations of \eqref{eq:classical_robust_optimization}.  In addition, the inclusion of an information-theoretic component in $ARMOR_D$ implies that it is also an extension of the general OT-DRO method \cite{bui_UDR_2022unified}, allowing for distribution neighborhoods which  have qualitatively different structure. This allows for qualitatively new types of robustness (i.e., beyond the shifting of samples within an allowed neighborhood), such as miss-specification of the mass of widely-separated modes. Algorithmically, this feature effectively introduces dynamical sample weights \eqref{eq:new_weights}, as shown via the formal solution \eqref{eq:Q_star_app}-\eqref{eq:Q_star_app}. In short,   adversarial samples in the $ARMOR_D$ method are both transported (as in other OT-based methods) and re-weighted; this principled dynamical reweighting mechanism is the new ingredient in our adversarial robustness methods and contributes non-trivially to the performance of the $ARMOR_D$ methods (see Section \ref{app:convex_comb_objective_functions}).

\section{Implementation Details}\label{app:implementation_details}
To foster reproducibility of our results, in addition to providing the source code at \url{https://github.com/star-ailab/ARMOR}, we provide the threat model, pseudocode for the OT-regularized-divergence adversarial robustness methods, and the target network structures used in the experiments.

\subsection{Threat Model} \label{app:threat_model}
Following the guidelines in \cite{carlini2019evaluating}, we consider the following threat model characterizing the adversary’s goal, knowledge, and capabilities in implementing $ARMOR_D$ as a method for enhancing adversarial robustness. 
\begin{enumerate}
    \item \textit{Adversaries goal:} The adversaries goal is to generate adversarial samples that force a  neural network model to make erroneous predictions. To avoid restrictive assumptions, any wrong classification is considered as a successful attack. 
    \item \textit{Adversary's knowledge:} To avoid restrictive assumptions, we assume that the adversary has complete knowledge of the inner workings of the target model (i.e., white-box access). This aligns with the Kerckhoff's principle that mandates security even if system details are known to the adversary.
    \item \textit{Adversary's Capabilities:} The adversary can apply arbitrary modifications to natural samples of any class.
\end{enumerate}

\subsection {Algorithm Pseudocode}
\label{app:algo}
In this section we provide pseudocode for the methods proposed in this work.  The pseudocode  will reference the following loss functions.
\begin{enumerate}
\item {\bf Inner Maximizer Objective for Adversarial Samples (see Eq.~\ref{eq:OT_regularized_loss_main}):}
\begin{align}
    &A_s(x,\tilde x,y,\lambda,\theta)\coloneqq \mathcal{L}_\theta(\tilde x,x,y)-\lambda c_s(x,\tilde x)\,.\label{eq:adv_sample_objective}
\end{align}
\item {\bf KL-Divergence Outer Minimizer Objective (see Eq.~\ref{eq:OT_reg_KL_DRO}):}
\begin{align}
    A_{w}^{KL}(x,\tilde x,y,\lambda,\theta)\coloneqq \epsilon\lambda+\lambda \log\left(\frac{1}{B}\sum_{i=1}^B\exp(\lambda^{-1}A_s(x_i,\tilde x_i,y_i,\lambda,\theta))\right)\,. \label{eq:KL_outer_min_objective}
\end{align}
\item {\bf $f$-Divergence Outer Minimizer Objective (see Eq.~\ref{eq:OT_reg_Df_DRO}):}
\begin{align}
A_{w}^f(x,\tilde x,y,\lambda,\rho,\theta)\coloneqq \epsilon\lambda+ \rho+\lambda \frac{1}{B}\sum_{i=1}^Bf^*((A_s(x_i,\tilde x_i,y_i,\lambda,\theta)-\rho)/\lambda)\,.\label{eq:f_div_outer_min_objective}
\end{align}
In the examples in Section \ref{sec:experiments}, the $f$-divergence experiments use use the $\alpha$-divergences, for which $f_\alpha^*$ is given in \eqref{eq:f_alpha_star}.
\item {\bf R{\'e}nyi-Divergence Outer Minimizer Objective (see Eq.~\ref{eq:OT_reg_Renyi_DRO}):}
\begin{align}
&A_w^{{R{\'e}n}}(x,\tilde{x},y,\lambda,\rho,\theta)\coloneqq \epsilon\lambda+\rho-\alpha^{-1}(\log\alpha +1)\lambda\label{eq:Renyi_div_outer_min_objective}\\
 &\qquad+ \begin{cases}
-\frac{1}{\alpha-1}\lambda\log\left(\frac{1}{B}\sum_{i=1}^B  (\lambda^{-1}(\rho-A_s(x_i,\tilde x_i,y_i,\lambda,\theta)))^{(\alpha-1)/\alpha} \right) & \text{ if } M_{\lambda,\theta}<\rho\\
\infty &\text{ if } M_{\lambda,\theta}\not < \rho
\end{cases}\,,\notag\\
&M_{\lambda,\theta}\coloneqq \max_i A_s(x_i,\tilde x_i,y_i,\lambda,\theta)\,.\notag
\end{align}
\end{enumerate}
The loss functions $\mathcal{L}_\theta$ in \eqref{eq:adv_sample_objective} come from \eqref{eq:PGD_T_M_losses} and  the cost function for the adversarial samples $c_s$ (i.e., having already incorporated the fixed label and natural sample constraints from the OT cost function) is given by:
\begin{enumerate}
    \item TRADES and MART use the $\epsilon$-ball hard-constraint cost  
\begin{align}
    c_\epsilon(x,\tilde x)=\infty 1_{\|x-\tilde x\|>\epsilon}\,.
\end{align}
\item UDR \cite{bui_UDR_2022unified}  uses the $\epsilon$-ball soft-constraint
\begin{align}\label{eq:c_UDR}
    c_{UDR}(x,\tilde{x})=\gamma\sum_i((x_i-\tilde{x}_i-\text{sign}(x_i-\tilde{x}_i) \epsilon)^21_{|x_i-\tilde{x}_i|\geq\epsilon}\,,
\end{align}
where $\gamma$ is another hyperparameter.
\end{enumerate}

\renewcommand{\algorithmicrequire}{\textbf{Input:}}
\renewcommand{\algorithmicensure}{\textbf{Output:}}
\begin{algorithm}
\caption{\textbf{A}dversarially \textbf{R}obust Deep Learning \textbf{M}odels with \textbf{O}ptimal-Transport-\textbf{R}egularized \textbf{D}ivergences ($ARMOR_D$)}\label{alg:adv_samples}
\begin{algorithmic}[1]
\Require Labeled training data $\{(x_i, y_i)\}$, target model $\phi_\theta$ depending on NN-parameters $\theta$, number of training epochs $N$, minibatch size $B$, information divergence $D$ ($D=KL$, $D=D_f$, or $D=R_\alpha$), number of inner maximizer iterations $M$, learning rates $lr_{\tilde x}$, $lr_\lambda$, and $lr_\theta$.
\Ensure robustified model $\phi_\theta$

\For{$n=1,\ldots, N$}
    \State Sample a minibatch $B_n$
    \For {$(x_i,y_i) \in B_n$}
        \State $\tilde{x}_i \gets x_i+$noise
        \For {$m=1, \ldots, M$}
            \State  $\tilde x_i\gets \tilde x_i+lr_{\tilde x}\nabla_{\tilde x} A_s(x_i,\tilde x_i,y_i,\lambda,\theta)$ \Comment {See \eqref{eq:adv_sample_objective}}
        \EndFor
    \EndFor
    \If{$D=KL$}
        \State $\lambda \gets \lambda-lr_\lambda \nabla_{\lambda} A^{KL}_w(x,\tilde x,y,\lambda,\theta)$  \Comment {See \eqref{eq:KL_outer_min_objective}}
        \State $\theta \gets \theta-lr_\theta \nabla_{\theta} A^{KL}_w(x,\tilde x,y,\lambda,\theta)$
    \ElsIf{$D=D_f$}
        \State $(\lambda,\rho) \gets (\lambda,\rho)-lr_\lambda \nabla_{(\lambda,\rho)} A^f_w(x,\tilde x,y,\lambda,\rho,\theta)$  \Comment {See \eqref{eq:f_div_outer_min_objective}}
        \State $\theta \gets \theta-lr_\theta \nabla_{\theta} A^f_w(x,\tilde x,y,\lambda,\rho,\theta)$
        \ElsIf{$D=R_\alpha$}
        \State $(\lambda,\rho) \gets (\lambda,\rho)-lr_\lambda \nabla_{(\lambda,\rho)} A^{R{\'e}n}_w(x,\tilde x,y,\lambda,\rho,\theta)$  \Comment {See \eqref{eq:Renyi_div_outer_min_objective}}
        \State $\theta \gets \theta-lr_\theta \nabla_{\theta} A^{R{\'e}n}_w(x,\tilde x,y,\lambda,\rho,\theta)$
    \EndIf
\EndFor
\end{algorithmic}
\end{algorithm}
Lines 3-8 of Algorithm~\ref{alg:adv_samples} implement the inner maximizer, wherein the adversarial samples $\tilde x_i$ are constructed, and lines 9-18 implement one step of the outer minimizer for the chosen divergence, $D$. In line 4 of the inner maximizer we allow for noise to be added to the natural sample when initializing the adversarial sample; default is IID Uniform$[-lr_{\tilde x},lr_{\tilde x}]$ noise added to each component. In addition, we also implemented and tested versions that replace the SGD step in line 10 with the bisection method and in lines 13 and 16 with  convex optimization routines; specifically, we use SciPy's \texttt{optimize.bisect} and \texttt{optimize.minimize} respectively.  Using these slightly increases the overall computational cost of the method, but results in improved performance in most cases and eliminates the need to tune the learning rate $lr_\lambda$. The modifications to Algorithm \ref{alg:adv_samples} (specifically, to the outer minimizers) that yield the $adv+nat$ and $adv^a$ methods are outlined in Appendices \ref{app:adv+nat_description} and \ref{app:asym_DRO} respectively below.

\subsection{Experimental Setup} 
\label{app:setup}
CIFAR-100, CIFAR-10, and MNIST datasets comprise 50,000 images in the training and 10,000 in the test set. While MNIST includes black-and-white images each with $28\times 28$ features, CIFAR-10 and CIFAR-100 include higher-dimensional colored images each with $32\times32\times3$ features. 

The target neural network for CIFAR-10 and CIFAR-100 is the widely-adopted ResNet18 \cite{bui_UDR_2022unified}.  For the MNIST dataset, we followed \cite{bui_UDR_2022unified} with the standard network architecture from \cite{carlini2017towards}. The MNIST image detector is a CNN network with four convolutional layers (32, 32, 64, and 64) each with a square filter of size 3 and ReLU activations, two $2 \times 2$ max-pooling layers, and three fully connected layers with a dropout layer between the first and second fully connected layer (\cite{carlini2017towards}). 

All attacks in our qualitative experiments with MNIST were conducted by the neighborhood size of $0.1$, $\ell_\infty$ neighborhood, 20 iterations for adversarial training, and 40 iterations for evaluation of adversarial robustness.  For the attacks in our benchmark experiments on CIFAR-10 and CIFAR-100, we followed the standard settings in \cite{bui_UDR_2022unified}. Accordingly, we used the $\ell_\infty$ neighborhood size of $8/255$ with the step size of $2/255$ for CIFAR-10; and the $\ell_\infty$ neighborhood size of $0.01$ with the step size of $0.001$ for CIFAR-100. Following the same standard, we used 10 iterations of adversarial training per sample and 200 iterations for the evaluation of the adversarial robustness.


\subsection{Robust Optimization Using a Mixture of Adversarial and Natural Samples}\label{app:adv+nat_description}
The best performance in the experiments presented in Section \ref{sec:experiments}  was often obtained using a mixture of adversarial samples along with the original training data (called the natural samples) and their corresponding losses.  This can be viewed as DRO over distribution neighborhoods of the form
\begin{align}\label{eq:adv+orig_nbhds}
    \mathcal{U}^{D^c}_{\epsilon,t}(P_n)\coloneqq \{tP_n+(1-t)Q:D^c(Q\|P_n)\leq \epsilon\}\,,\,\,\,\epsilon>0,t\in(0,1)\,,
\end{align}
as we have
\begin{align}\label{eq:DRO_adv+orig}
 \inf_{\theta\in \Theta}\sup_{Q\in\mathcal{U}^{D^c}_{\epsilon,t}(P_n)} E_Q[\mathcal{L}_\theta]=  \inf_{\theta\in\Theta}\left\{tE_{P_n}[\mathcal{L}_\theta]+(1-t)\sup_{Q:D^c(Q\|P_n)\leq \epsilon}E_Q[\mathcal{L}_\theta]\right\}\,.
\end{align}
The supremum over $Q$ on the right-hand side of \eqref{eq:DRO_adv+orig} can then be evaluated by the method discussed in Section \ref{sec:formal_derivation} and the resulting expression is used in what we call the $adv_s+nat$ methods.

\subsection{Asymmetric Robust Optimization}\label{app:asym_DRO}
In many cases the training samples are naturally partitioned into distinct components, with corresponding empirical distributions $P_{n,0}$ and  $P_{n,1}$ (e.g., distinct class labels), and one wishes to robustify only one component of the partition (e.g., to protect against false negative adversarial attacks but not false positives).    In such cases one can formulate the DRO problem in an asymmetric manner as follows. Define the baseline distribution $P_{n,s}=(1-s)P_{n,0}+sP_{n,1}$ for some $s\in(0,1)$ and  define the distribution neighborhoods
\begin{align}
 \mathcal{U}^{a,D^c}_{\epsilon}(P_{n,s})\coloneqq\{(1-s)P_{n,0}+sQ:D^c(Q\|P_{n,1})\leq \epsilon\}\,.   
\end{align}
The corresponding DRO problem can  be rewritten as
\begin{align}\label{eq:DRO_asym}
  \inf_{\theta\in\Theta}\sup_{Q\in \mathcal{U}^{a,D^c}_{\epsilon}(P_{n,s})} E_Q[\mathcal{L}_\theta]=\inf_{\theta\in\Theta}\left\{(1-s)E_{P_{n,0}}[\mathcal{L}_\theta]+s\sup_{Q:D^c(Q\|P_{n,1})\leq \epsilon} E_Q[\mathcal{L}_\theta]\right\}\,,  
\end{align}
where one can clearly see that the objective on the right-hand side is non-robust in $P_{n,0}$ but uses OT-regularized-divergence robust optimization for the $P_{n,1}$ component.  The parameter $s$  weights the relative importance of the partition components in the overall loss; it can be chosen to correspond to the relative sizes of the partition components (i.e., so that $P_{n,s}=P_n$) or it can be used as a hyperparameter. The supremum over $Q$ on the right-hand side of \eqref{eq:DRO_asym} can  be evaluated as in Section \ref{sec:formal_derivation},  resulting in what we call the $adv^a$ methods. This approach can  easily be extended to partitions with more than two components.

{\bf Asymmetric Robust Optimization Using a Mixture of  Adversarial and Natural Samples:} One can  combine  asymmetry with the use of natural samples. To do this, choose a mixing parameter $t\in(0,1)$ and define the distribution neighborhoods
\begin{align}
\mathcal{U}^{a,D^c}_{\epsilon,t}(P_{n,s})\coloneqq&\{tP_{n,s}+(1-t)((1-s)P_{n,0}+sQ):D^c(Q\|P_{n,1})\leq \epsilon\}\\
=& \{(1-s)P_{n,0}+tsP_{n,1}+(1-t)sQ:D^c(Q\|P_{n,1})\leq \epsilon\}\,.\notag
\end{align}
The corresponding DRO problem can  be rewritten as
\begin{align}\label{eq:DRO_asym+nat}
  &\inf_{\theta\in\Theta}\sup_{Q\in \mathcal{U}^{a,D^c}_{\epsilon,t}(P_{n,s})} E_Q[\mathcal{L}_\theta]\\
  =&\inf_{\theta\in\Theta}\left\{(1-s)E_{P_{n,0}}[\mathcal{L}_\theta]+tsE_{P_{n,1}}[\mathcal{L}_\theta]+(1-t)s \sup_{Q:D^c(Q\|P_{n,1})\leq \epsilon} E_Q[\mathcal{L}_\theta]\right\}\,.\notag
\end{align}
Once again, the supremum over $Q$ on the right-hand side can  be evaluated as in Section \ref{sec:formal_derivation}, resulting in what we call the $adv^a+nat$ methods.

\subsection{Interpolating between  OT-Regularized-$D_f$ and OT Methods}\label{app:convex_comb_objective_functions}
Here we describe a general procedure for modifying an $f$-divergence into a one-parameter family, $D_{f_\beta}$, so that the resulting OT-regularized-$D_{f_\beta}$ method  interpolates between OT DRO and  OT-regularized-$D_f$ DRO.  In particular, this enables us to examine the effect of "turning off" the information divergence component of the method, i.e., the adversarial sample weights (see Section \ref{sec:reweighting}).   This is a distinct interpolation procedure from the result in Theorem \ref{thm:limit_D_c_to_C}.

Given an $f$-divergence, $D_f$, and $\beta\in(0,1]$ define 
\begin{align}
    f_\beta(t)=\beta f((t-1+\beta)/\beta)\,,
\end{align}
(not to be confused with the $\alpha$-divergences, Eq.~\ref{eq:f_alpha_def}).  The $f_\beta$ are convex and $f_\beta(1)=0$ for all $\beta$, hence $D_{f_\beta}$ is a well-defined family of divergences with $\beta=1$ giving the original $f$-divergence.  It is straightforward to compute the Legendre transform of $f_\beta$ in terms of that of $f$,
 \begin{align}
     f_\beta^*(t)=\beta f^*(t)+(1-\beta)t\,.
 \end{align}
 Therefore one can use $f_\beta$ to define an OT-regularized-divergence DRO problem and simplify it as follows
 \begin{align}\label{eq:DRO_method_mixture}
&\inf_{\theta\in\Theta}\sup_{Q: D_{f_\beta}^c(Q\|P_n)\leq \epsilon}E_Q[\mathcal{L}_\theta]\\
=&\inf_{\lambda>0,\rho\in\mathbb{R},\theta\in\Theta}\left\{ \epsilon\lambda+ \beta(\rho+\lambda   \frac{1}{n}\sum_{i=1}^nf^*((\mathcal{L}_{\theta,\lambda}^{c}(z_i)-\rho)/\lambda))+(1-\beta) \frac{1}{n}\sum_{i=1}^n\mathcal{L}_{\theta,\lambda}^{c}(z_i)\right\}\,.\notag
\end{align}
As $\beta\to 0^+$ the objective function in \eqref{eq:DRO_method_mixture} approaches that of OT-DRO and so \eqref{eq:DRO_method_mixture} can be thought of as a mixture of OT DRO and OT-regularized-$D_f$ DRO. Moreover, the mixing parameter $\beta$ sets a lower bound on the adversarial sample weights \eqref{eq:new_weights}, with $p_{*,i}\geq(1-\beta)/n$.  In the KL case one can  evaluate the infimum over $\rho$ in \eqref{eq:DRO_method_mixture} to obtain
  \begin{align}\label{eq:DROKL_method_mixture}
&\inf_{\theta\in\Theta}\sup_{Q: KL_\beta^c(Q\|P_n)\leq \epsilon}E_Q[\mathcal{L}_\theta]\\
=&\inf_{\lambda>0,\theta\in\Theta}\left\{ \epsilon\lambda+ \beta\lambda   \log\left(\frac{1}{n}\sum_{i=1}^n\exp(\lambda^{-1}\mathcal{L}_{\theta,\lambda}^{c}(z_i))\right)+(1-\beta) \frac{1}{n}\sum_{i=1}^n\mathcal{L}_{\theta,\lambda}^{c}(z_i)\right\}\,.\notag
\end{align}
When $D_f$ is an $\alpha$-divergence we denote the  method \eqref{eq:DRO_method_mixture}  by $ARMOR_{\alpha,\beta}$. We denote the method \eqref{eq:DROKL_method_mixture} by $ARMOR_{KL_\beta}$. In our tests on MNIST we found that the performance degrades significantly as $\beta$ decreases to $0$, both in terms of adversarial robustness and performance generalizability  (see Figure~\ref{fig:beta}). This implies that the information divergence, and the adversarial sample weights which it generates,  contributes non-trivially to the success of the method.

\begin{figure}[H]
    \centering
    \includegraphics[width=0.65\textwidth]{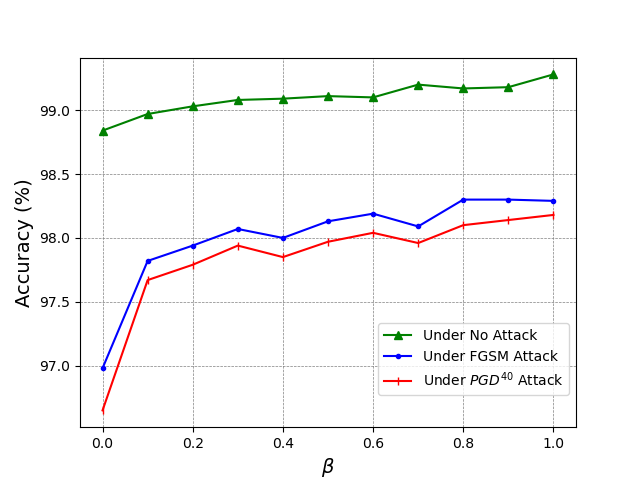}
    \caption{$ARMOR_{\alpha,\beta}$-robustified model performance against $\beta$.}
    \label{fig:beta}
\end{figure}

\bibliographystyle{siamplain}
\bibliography{armor_arxiv.bbl}

\end{document}